\documentclass[11pt]{article}

\usepackage[final]{acl}

\usepackage{times}
\usepackage{latexsym}
\usepackage{amsmath}
\usepackage{amsthm}
\usepackage{amsfonts}
\usepackage{mathtools}
\usepackage{booktabs}
\usepackage{todonotes}
\usepackage{pgfplots}
\usepackage{makecell}
\usepackage{xcolor}
\usepackage{hyperref}
\usepackage[most]{tcolorbox}
\usepackage{standalone}
\pgfplotsset{compat=1.18}
\pgfplotsset{unbounded coords=jump} 
\usepackage{tikz}
\usepackage{amssymb}
\usetikzlibrary{arrows.meta,positioning,calc,fit,shapes,patterns,decorations.pathreplacing,patterns}
\usepackage{subcaption}

\usepackage[T1]{fontenc}

\usepackage[utf8]{inputenc}

\usepackage{microtype}

\usepackage{inconsolata}

\usepackage{graphicx}
\definecolor{judgegreen}{RGB}{0, 85, 98}

\newtcolorbox{judgebox}[1][]{
    colback=white,         
    colframe=black,   
    colbacktitle=black,
    coltitle=white,        
    boxrule=1.0pt,         
    arc=0pt,               
    outer arc=0pt,
    fonttitle=\bfseries\large,
    title={GPT-4o Judge Prompt}, 
    enhanced,               
    attach boxed title to top left={xshift=0pt, yshift=0pt}, 
    boxed title style={frame hidden}, 
    #1 
}
\newtheoremstyle{uprightthm}
  {3pt}
  {3pt}
  {\normalfont}
  {}
  {\bfseries}
  {.}
  {.5em}
  {}

\theoremstyle{uprightthm}
\newtheorem{definition}{Definition}
\newtheorem{lemma}{Lemma}

\newtheorem{proposition}{Proposition}

\usetikzlibrary{calc, positioning, patterns, decorations.pathreplacing, arrows.meta, shapes.geometric, backgrounds, fit, shadows}
\usepackage{geometry}

\definecolor{myblue}{RGB}{31, 119, 180}
\definecolor{mybluefill}{RGB}{214, 230, 244}
\definecolor{myorange}{RGB}{255, 127, 14}
\definecolor{myorangefill}{RGB}{255, 230, 209} 
\definecolor{mygreen}{RGB}{44, 160, 44}
\definecolor{myred}{RGB}{214, 39, 40}
\definecolor{mygray}{RGB}{100, 100, 100}

\title{QuantileMark: A Message-Symmetric Multi-bit Watermark for LLMs}

\author{
  \textbf{Junlin Zhu},
  \textbf{Baizhou Huang},
  \textbf{Xiaojun Wan} \\
  \textsuperscript{}Wangxuan Institute of Computer Technology, Peking University \\
  \texttt{zhujunlin@stu.pku.edu.cn, \{hbz19, wanxiaojun\}@pku.edu.cn} \\
}

\begin{document}
\maketitle
\begin{abstract}
As large language models become standard backends for content generation, practical provenance increasingly requires multi-bit watermarking. 
In provider-internal deployments, a key requirement is message symmetry: the message itself should not systematically affect either text quality or verification outcomes.
Vocabulary-partition watermarks can break message symmetry in low-entropy decoding: some messages are assigned most of the probability mass, while others are forced to use tail tokens. 
This makes embedding quality and message decoding accuracy message-dependent.
We propose QuantileMark, a white-box multi-bit watermark that embeds messages within the continuous cumulative probability interval $[0, 1)$.
At each step, QuantileMark partitions this interval into $M$ equal-mass bins and samples strictly from the bin assigned to the target symbol, ensuring a fixed $1/M$ probability budget regardless of context entropy.
For detection, the verifier reconstructs the same partition under teacher forcing, computes posteriors over latent bins, and aggregates evidence for verification.
We prove message-unbiasedness, a property ensuring that the base distribution is recovered when averaging over messages. 
This provides a theoretical foundation for generation-side symmetry, while the equal-mass design additionally promotes uniform evidence strength across messages on the detection side.
Empirical results on C4 continuation and LFQA show improved multi-bit recovery and detection robustness over strong baselines, with negligible impact on generation quality. Our code is available at \href{https://github.com/zzzjunlin/QuantileMark}{GitHub}.
\end{abstract}

\section{Introduction}
\label{sec:intro}

Large language models (LLMs) have become standard backends for applications ranging from dialogue assistance and content creation to code generation and data analysis~\cite{chatgpt}.
As high-quality synthetic text becomes ubiquitous and inexpensive, provenance has emerged as a practical necessity: platforms and providers must attribute the origin of content and its deployment settings.
This capability is critical not only for mitigating misinformation but also for policy enforcement, abuse response (e.g., toxic content), copyright disputes, and enterprise compliance~\cite{yang2025survey}. 

Generative watermarking addresses this need by embedding a covert signal during the decoding process, enabling a detector to verify the model's authorship~\cite{kirchenbauer2023watermark}. 
While early work focused on \emph{zero-bit} schemes for binary presence detection, real-world deployments increasingly demand \emph{multi-bit} provenance to encode metadata such as user IDs, model versions, or timestamps~\cite{yoo2023advancing}. 
Multi-bit attribution is inherently \textbf{provider-internal}. The entity hosting the model must also control the detection service. 
After all, even if a third party extracts an embedded identifier, they cannot map it to a real-world identity without the provider's private user database.
Consequently, the provider is the ultimate root of trust for identity resolution. 
This dynamic practically justifies white-box access to the model at detection time.

Unlike zero-bit watermarking, where the detector only answers a binary question of presence, multi-bit provenance requires decoding a specific message (e.g., a user ID).
We represent the provenance message as a binary string segmented into $m$-bit symbols, so each symbol takes one of $M=2^m$ discrete values.
At each watermarking step, we embed one symbol and can convey up to $m$ bits of information.
For example, a 24-bit message can be encoded as 12 symbols when $m=2$.

The process of encoding these messages introduces a critical new requirement: \textbf{message symmetry}. 
Message symmetry means that the message itself should not systematically affect either text quality or verification outcomes.
In particular, this concern arises in two places.
During generation, messages should receive comparable probability-mass budgets so that text quality does not depend on the identifier.
During detection, evidence should accumulate comparably across messages so that watermark detection are not message-dependent.
Ensuring this symmetry is essential for operational fairness, as it guarantees that the provider can deliver consistent service quality and attribution reliability across the entire user base.


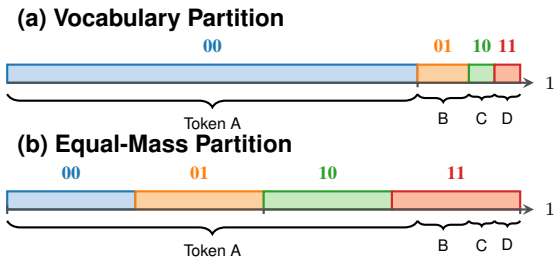
\begin{figure}[t]
  \centering
  \begin{tikzpicture}[font=\sffamily\scriptsize, xscale=0.9, yscale=0.60]
    \def\w{7.5}        
    \def\h{0.4}        
    \def\gap{2.8}      
    \def\tickH{0.1}    
    \def\braceY{-0.25} 
    \def\labelOffset{0.05} 

    \definecolor{sym00}{RGB}{31, 119, 180}
    \definecolor{sym00fill}{RGB}{198, 219, 239}
    \definecolor{sym01}{RGB}{255, 127, 14}
    \definecolor{sym01fill}{RGB}{253, 208, 162}
    \definecolor{sym10}{RGB}{44, 160, 44}
    \definecolor{sym10fill}{RGB}{199, 233, 192}
    \definecolor{sym11}{RGB}{214, 39, 40}
    \definecolor{sym11fill}{RGB}{252, 187, 161}
    
    \definecolor{mygray}{RGB}{80, 80, 80}

    \node[anchor=west, font=\sffamily\small\bfseries] at (0, \h+1.0) {(a) Vocabulary Partition};

    \draw[fill=sym00fill, draw=sym00, thick] (0,0) rectangle (0.8*\w, \h);
    
    \draw[fill=sym01fill, draw=sym01, thick] (0.8*\w,0) rectangle (0.9*\w, \h);
    
    \draw[fill=sym10fill, draw=sym10, thick] (0.9*\w,0) rectangle (0.95*\w, \h);

    \draw[fill=sym11fill, draw=sym11, thick] (0.95*\w,0) rectangle (1.0*\w, \h);

    \node[text=sym00, font=\bfseries\scriptsize, anchor=south] at (0.4*\w, \h+\labelOffset) {00};
    \node[text=sym01, font=\bfseries\scriptsize, anchor=south] at (0.85*\w, \h+\labelOffset) {01};
    \node[text=sym10, font=\bfseries\scriptsize, anchor=south] at (0.925*\w, \h+\labelOffset) {10};
    \node[text=sym11, font=\bfseries\scriptsize, anchor=south] at (0.975*\w, \h+\labelOffset) {11};

    \draw[->, thick, >=stealth, mygray] (0,0) -- (\w+0.2, 0) node[right, text=black] {$1$};
    \draw[thick, mygray] (0.8*\w, 0) -- (0.8*\w, -\tickH); 
    
    \draw[decorate, decoration={brace, amplitude=5pt, mirror}, thick, black] 
        (0, \braceY) -- (0.8*\w, \braceY) 
        node[midway, below=5pt, font=\sffamily\tiny] {Token A };
    
    \draw[decorate, decoration={brace, amplitude=3pt, mirror}, thick, black] 
        (0.8*\w, \braceY) -- (0.9*\w, \braceY) 
        node[midway, below=3pt, font=\sffamily\tiny] {B};

    \draw[decorate, decoration={brace, amplitude=2pt, mirror}, thick, black] 
        (0.9*\w, \braceY) -- (0.95*\w, \braceY);
    \node[below=3pt, font=\sffamily\tiny] at (0.925*\w, \braceY) {C};

    \draw[decorate, decoration={brace, amplitude=2pt, mirror}, thick, black] 
        (0.95*\w, \braceY) -- (1.0*\w, \braceY);
    \node[below=3pt, font=\sffamily\tiny] at (0.975*\w, \braceY) {D};

    \begin{scope}[yshift=-\gap cm]
        \node[anchor=west, font=\sffamily\small\bfseries] at (0, \h+1.0) {(b) Equal-Mass Partition};

        \draw[fill=sym00fill, draw=sym00, thick] (0,0) rectangle (0.25*\w, \h);
        
        \draw[fill=sym01fill, draw=sym01, thick] (0.25*\w,0) rectangle (0.5*\w, \h);
        
        \draw[fill=sym10fill, draw=sym10, thick] (0.5*\w,0) rectangle (0.75*\w, \h);
        
        \draw[fill=sym11fill, draw=sym11, thick] (0.75*\w,0) rectangle (1.0*\w, \h);

        \node[text=sym00, font=\bfseries\scriptsize, anchor=south] at (0.125*\w, \h+\labelOffset) {00};
        \node[text=sym01, font=\bfseries\scriptsize, anchor=south] at (0.375*\w, \h+\labelOffset) {01};
        \node[text=sym10, font=\bfseries\scriptsize, anchor=south] at (0.625*\w, \h+\labelOffset) {10};
        \node[text=sym11, font=\bfseries\scriptsize, anchor=south] at (0.875*\w, \h+\labelOffset) {11};

        \draw[->, thick, >=stealth, mygray] (0,0) -- (\w+0.2, 0) node[right, text=black] {$1$};
        \draw[thick, mygray] (0, \tickH) -- (0, -\tickH);
        \draw[thick, mygray] (0.5*\w, 0) -- (0.5*\w, -\tickH); 
        \draw[thick, mygray] (1.0*\w, 0) -- (1.0*\w, -\tickH);

        \draw[decorate, decoration={brace, amplitude=5pt, mirror}, thick, black] 
            (0, \braceY) -- (0.8*\w, \braceY) 
            node[midway, below=5pt, font=\sffamily\tiny] {Token A};
        
        \draw[decorate, decoration={brace, amplitude=3pt, mirror}, thick, black] 
            (0.8*\w, \braceY) -- (0.9*\w, \braceY) 
            node[midway, below=3pt, font=\sffamily\tiny] {B};

        \draw[decorate, decoration={brace, amplitude=2pt, mirror}, thick, black] 
            (0.9*\w, \braceY) -- (0.95*\w, \braceY);
        \node[below=3pt, font=\sffamily\tiny] at (0.925*\w, \braceY) {C};

        \draw[decorate, decoration={brace, amplitude=2pt, mirror}, thick, black] 
            (0.95*\w, \braceY) -- (1.0*\w, \braceY);
        \node[below=3pt, font=\sffamily\tiny] at (0.975*\w, \braceY) {D};

    \end{scope}

  \end{tikzpicture}
  \caption{Impact of probability mass partitioning strategies on message symmetry under a low-entropy distribution. (a) Vocabulary Partition allocates uneven probability budgets to messages that should ideally be weighted equally. (b) Equal-Mass Partition assigns a fixed probability budget to each message.}
  \label{fig:partition_comparison}
\end{figure}

Existing generative watermarks can be broadly grouped into \emph{distortionary} and \emph{distribution-preserving} designs.
Both design lines face challenges when message symmetry is required for multi-bit provenance.

Distortionary schemes typically employ a \textbf{vocabulary partition} to construct the signal~\cite{yoo2023advancing}.
For a message space of size $M$, these methods partition the vocabulary into $M$ disjoint sets and apply a logit bias to the subset corresponding to the target symbol.
This approach, however, introduces severe asymmetry under low-entropy decoding.
Consider a step where the model assigns probabilities \{0.8, 0.1, 0.05, 0.05\} to its top candidates, as shown in Figure~\ref{fig:partition_comparison}.
If the token with probability 0.8 falls into the partition set for one value, embedding that value is almost free: the model outputs its preferred token and the detector observes the signal with high probability.
Conversely, for the other $M-1$ values, their assigned sets may contain only low-probability tail tokens.
Embedding these values forces the model to override its natural choice, which degrades text quality while simultaneously yielding weaker statistical evidence.
Consequently, the expected evidence averaged over all messages is diluted.


Distribution-preserving methods ensure the watermarked output preserves the base distribution when marginalized over randomness.
This goal can already mitigate message-dependent distortion on the generation side.
However, constrained by the requirement of \textbf{black-box verification}, these schemes typically prioritize stealth through \emph{minimal stepwise distortion}~\cite{jiang2025stealthink}.
This creates a fundamental bottleneck for multi-bit provenance: the evidence is often too subtle for robust message recovery.
In provider-internal settings with white-box access, such constraints are unnecessary.
A more attractive direction is to enforce message symmetry explicitly, then exploit its benefit: allocate each symbol a fixed probability budget and permit larger, controlled deviations, ensuring effective evidence accumulation.

To overcome the structural asymmetry of vocabulary partitioning and the evidence limitations of existing distribution-preserving designs, we introduce \textbf{QuantileMark}, a white-box multi-bit watermark based on continuous probability mass partitioning.
At each step, QuantileMark partitions the cumulative probability interval $[0,1)$ into $M$ bins of equal probability mass and encodes a symbol by sampling within the corresponding bin (Figure~\ref{fig:framework}).
During detection, the verifier reconstructs the same partition to compute posteriors for the latent bins.
Equal-mass allocation enforces message symmetry by construction, and it turns symmetry into stronger and more stable stepwise evidence regardless of context entropy.

In summary, our contributions are threefold. 
First, we introduce QuantileMark, a provenance framework that guarantees a uniform probability budget for every message symbol via quantile partitioning, accompanied by a white-box detector that computes posteriors for decoding.
Second, we formalize the notion of message-unbiasedness and prove that QuantileMark satisfies this property, providing a theoretical guarantee for generation-side symmetry.
Finally, we demonstrate that QuantileMark outperforms vocabulary-based baselines on C4 and LFQA in recovery and robustness, while maintaining generation quality.

\begin{figure*}[t]
  \centering

  \begin{subfigure}[t]{\textwidth}
    \caption{Embedding Stage}
    \label{fig:framework:a}
    \centering
    \begin{tikzpicture}[
    yscale=0.95, 
    transform shape,
    font=\sffamily,
    >=Stealth,
    tokenbox/.style={draw=gray!80, thick, minimum height=0.9cm, inner sep=0pt, align=center, fill=white},
    selectedtoken/.style={draw=myorange, very thick, minimum height=0.9cm, inner sep=0pt, align=center, fill=myorangefill},
    binbox/.style={draw=myblue!50, thick, minimum height=0.9cm, inner sep=0pt, align=center, fill=white, text=gray},
    selectedbin/.style={draw=myblue, very thick, minimum height=0.9cm, inner sep=0pt, align=center, fill=mybluefill, text=myblue},
    logicbox/.style={draw=mygreen!80!black, thick, rounded corners=3pt, fill=mygreen!5, align=center, font=\sffamily\scriptsize, inner sep=4pt, drop shadow={opacity=0.15}},
    inputbox/.style={draw=gray!60, rounded corners=2pt, fill=white, font=\sffamily\small, align=center, drop shadow={opacity=0.15}, inner sep=3pt},
    stepnode/.style={circle, fill=black!70, text=white, font=\bfseries\tiny, inner sep=1.5pt},
    overlapfill/.style={pattern=north east lines, pattern color=myblue!80},
    pinstyle/.style={fill=myred, draw=white, line width=0.8pt}
]

    \begin{scope}[local bounding box=generation_stage]
        \def\geomStart{2.5}        
        \def\barWidth{9.6}        
        \def\topBarY{2.2}         
        \def\botBarY{0.0}        
        \def\boxH{0.65}            
        \def\binW{2.4}            

        
        \node[logicbox, minimum width=2.0cm, minimum height=1.0cm] 
            (randomness) at (0.2, 0.0) { 
           Symbol Index $i$\\[0pt]
        \& Permutation $\phi$
        };

        \node[logicbox, above=0.7cm of randomness, minimum width=2.8cm, align=center] 
            (selection) {
            Map Target Symbol $s_{i}$\\[0pt]
            to Bin $B_{r^\star}$
        };

        \node[inputbox, 
            above=0.35cm of selection, 
            align=center, 
            inner sep=3pt
        ] (msg) {
            \textbf{Message} $S$\\[2pt]
            $[ s_1, \textcolor{myblue}{\mathbf{s_2{=}\texttt{10}}}, s_3, s_4 ]$
        };

        \node[inputbox, left=0.4cm of randomness.170, anchor=east] (ctx) {Local Context $$h$$};
        \node[inputbox, left=0.4cm of randomness.200, anchor=east] (key) {Key};

        \node[stepnode, anchor=south east] at (randomness.north west) {1};
        \node[stepnode, anchor=south east] at (selection.north west) {2};

        \draw[->, thick, gray] (ctx.east) -- (randomness.170);
        \draw[->, thick, gray] (key.east) -- (randomness.200);
        
        \draw[->, thick, mygreen!80!black] (randomness.north) -- (selection.south);
        \draw[->, thick, myblue] (msg.south) -- (selection.north);

        \begin{scope}[shift={(\geomStart, 0)}]
            \node[font=\scriptsize\bfseries, text=gray] at (0, \botBarY-0.55) {0};
            \node[font=\scriptsize\bfseries, text=gray] at (\barWidth, \botBarY-0.55) {1};
            \draw[gray!30, thick] (0, \botBarY-0.1) -- (\barWidth, \botBarY-0.1);

            \node[anchor=south, font=\bfseries] at (\barWidth/4, \topBarY + \boxH + 0.05) {Token Probability Intervals};
            
            \filldraw[tokenbox] (0, \topBarY) rectangle ++(5.2, \boxH) node[midway, xshift=-1.0cm] {``has''};
            \fill[overlapfill] (4.8, \topBarY) rectangle (5.2, \topBarY+\boxH); 
            
            \filldraw[selectedtoken] (5.2, \topBarY) rectangle ++(2.4, \boxH) node[midway, xshift=0.3cm] {\textbf{``is''}};
            \fill[overlapfill] (5.2, \topBarY) rectangle (7.2, \topBarY+\boxH); 
            
            \filldraw[tokenbox] (7.6, \topBarY) rectangle ++(1.2, \boxH) node[midway] {``in''};
            \filldraw[tokenbox, fill=gray!10] (8.8, \topBarY) rectangle ++(0.8, \boxH) node[midway, font=\tiny] {...};

            \node[anchor=north, align=center,font=\bfseries] at (\barWidth/4, \botBarY - 0.25) {
                \textbf{Quantile Bins}
            };
            
            \draw[binbox] (0, \botBarY) rectangle ++(\binW, \boxH) node[midway] {$B_0$};
            \draw[binbox] (\binW, \botBarY) rectangle ++(\binW, \boxH) node[midway] {$B_1$};
            \filldraw[selectedbin] (2*\binW, \botBarY) rectangle ++(\binW, \boxH) node[midway, font=\bfseries] {$B_2$};
            \draw[binbox] (3*\binW, \botBarY) rectangle ++(\binW, \boxH) node[midway] {$B_3$};

            \def\bStart{4.8} 
            \def\bEnd{7.2}   
            \draw[dashed, myblue, thick] (\bStart, \botBarY+\boxH) -- (\bStart, \topBarY);
            \draw[dashed, myblue, thick] (\bEnd, \botBarY+\boxH) -- (\bEnd, \topBarY);

            \def\uPos{6.7} 
            
            \coordinate (u_point) at (\uPos, 0.325); 
            \coordinate (pin_head) at (\uPos, 0.35); 

            \draw[|-|, myred!60, thick] (4.8, -0.35) -- (7.2, -0.35);
            \node[font=\small\bfseries, text=myred!80, below] at (6.0, -0.35) {Bin-Restricted Sample};

            \draw[myred, thick] (u_point) -- (pin_head);
            \filldraw[pinstyle] (pin_head) circle (3.5pt);
            \node[right=1pt, font=\scriptsize\bfseries, text=myred] at (u_point) {$u \sim U(B_2)$};

            \draw[->, densely dashed, very thick, myred] (pin_head) -- (\uPos, \topBarY);

            \filldraw[fill=myred, draw=white, thick] (\uPos, \topBarY) circle (2.2pt);

            \node[align=left, font=\sffamily\small, text=myred, anchor=west] at (\uPos+0.35, 1.4) {
                Token covers $u$\\
                $\Rightarrow$ Output \textbf{``is''}
            };

            \node[stepnode, anchor=south east] at (4.35, -0.6) {3}; 
            \node[stepnode, anchor=south west] at (\uPos+0.8, 1.8) {4}; 

        \end{scope}

        \draw[->, very thick, myblue] (selection.east) .. controls +(0.8,0) and +(-1.5, 0.1) .. 
                (\geomStart + 2.5*\binW, \botBarY + \boxH + 0.2);
        
        \node[text=myblue, font=\scriptsize\bfseries, fill=white, inner sep=1pt] at (\geomStart + 1.2, 1.0) {Target $B_{r^\star}$};

    \end{scope} 
\end{tikzpicture}
  \end{subfigure}

  \vspace{0.0em}

  \begin{subfigure}[t]{\textwidth}
    \caption{Detection Stage}
    \label{fig:framework:b}
    \centering
    \input{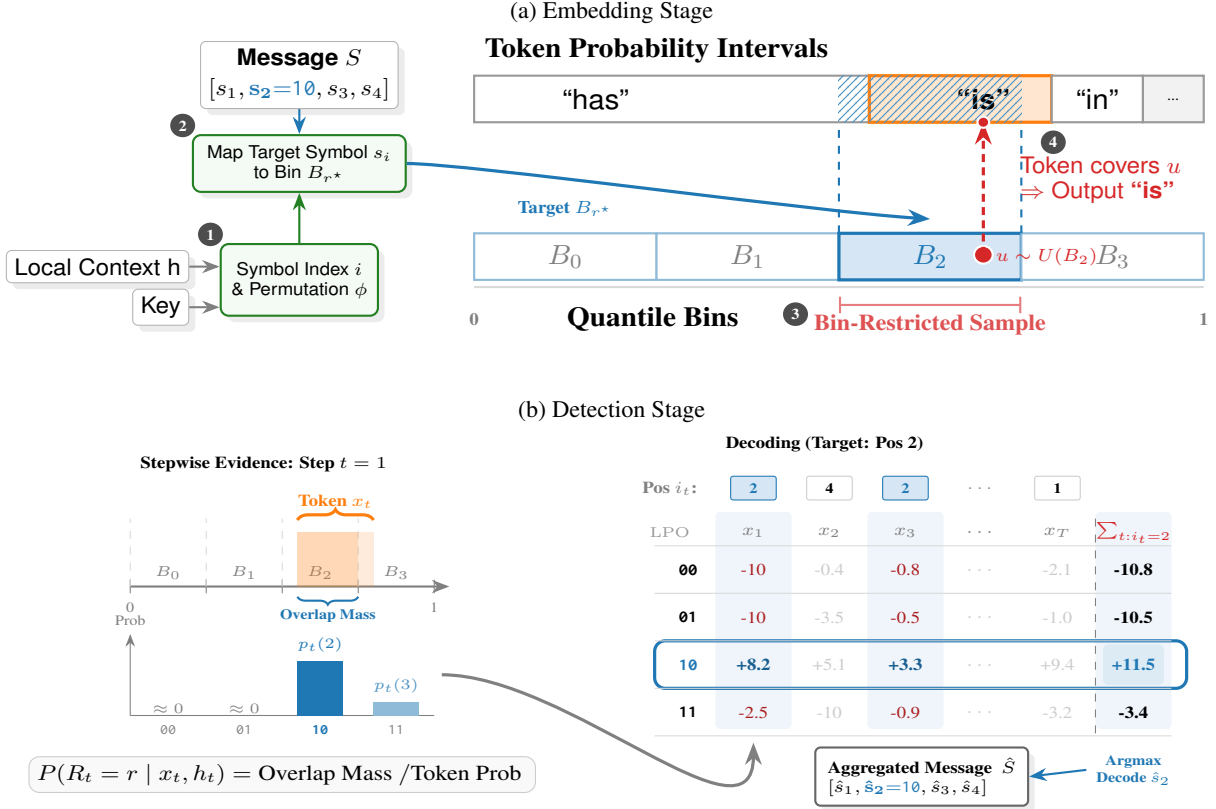}
  \end{subfigure}

  \caption{Overview of QuantileMark, where $m=2$, $H=4$. (a) \textbf{Embedding:} The key-derived logic (green) maps symbol to target bin $B_{r^\star}$ (blue frame). A random value $u$ (red dashed line) is sampled uniformly within this bin, forcing the selection of the token \textit{is} whose interval covers $u$. (b) \textbf{Detection:} Posteriors (blue bars) form the evidence $\mathrm{LPO}$ at this step. These scores are aggregated across steps assigned to position $i_t=2$ (right, blue columns) to decode $\hat{s}_2$. }
  \label{fig:framework}
\end{figure*}

\section{Preliminaries}
\label{sec:preliminaries}

We study a provider-internal watermarking setting where a model generates a sequence $x_{1:T}$ using a fixed secret key $K$.
Let $h_t \coloneqq x_{<t}$ denote the history at step $t$, and let $p_0(\cdot\mid h_t)$ denote the base next-token distribution.
Our goal is to embed a message $S = (s_1, \dots, s_H)$ consisting of $H$ discrete symbols.
Each symbol $s_h$ encodes $m$ bits of information and is drawn from the set $[M] \coloneqq \{0, \dots, M-1\}$, where $M=2^m$.

\textbf{Position Allocation.}
Unlike zero-bit presence detection, recovering a multi-bit message requires locating where each symbol is embedded.
A standard abstraction is \emph{position allocation}~\cite{yoo2023advancing}: a keyed function maps each step $t$ to an index $i_t\in\{1,\dots,H\}$.
This implies that step $t$ carries evidence \emph{only} about the specific message symbol $s_{i_t}$.

\textbf{Vocabulary Partitioning as Multi-Bit Channel.}
Many existing methods construct the channel by partitioning the vocabulary.
At step $t$, the vocabulary is split into $M$ disjoint sets $\{\mathcal{V}^{(r)}_t\}_{r=0}^{M-1}$.
To embed symbol $s_{i_t}$, the sampler promotes tokens in $\mathcal{V}^{(s_{i_t})}_t$.
In this paradigm, the abstract message symbol is tied directly to discrete token identities.


\textbf{CDF and Quantile Function View of Sampling.}
It is convenient to represent discrete sampling using the cumulative distribution function (CDF) $F$ and its inverse, the quantile function $Q(u) = F^{-1}(u)$.
Under an ordering of tokens at step $t$, each token corresponds to a specific interval on the cumulative probability interval $[0,1)$ whose length equals its probability.
Standard sampling is operationally equivalent to drawing $u\sim\mathrm{Unif}[0,1)$ and selecting the output token via $x = Q(u)$, \textbf{which returns the token whose interval contains $u$}.
This geometric view allows us to decouple the channel from discrete token identities by defining the watermark logic directly on the continuous domain of the quantile function.
We refer to embedding strategies that operate on this continuous interval as \textbf{quantile watermarking}.

\section{Methodology: Equal-Mass Quantile Watermarking}
\label{sec:methodology}

We propose \textbf{QuantileMark}, a white-box watermarking framework designed to ensure \textbf{message symmetry} by defining the embedding channel (i.e., the medium where the watermark signal is injected) on the cumulative probability interval rather than vocabulary space.

The framework consists of two procedures (Figure~\ref{fig:framework}).
During generation, QuantileMark partitions the stepwise cumulative probability interval into $M$ equal-mass quantile bins and samples a token strictly from the bin assigned to the target message symbol.
During detection, the verifier reconstructs this quantile geometry, computing the posterior probability of the latent bin to aggregate evidence and decode the message.

\subsection{Setup: Message and Position Allocation}
\label{sec:setup-pa}

Let $x_{1:T}$ be the generated sequence and $h_t$ the history at step $t$.
We embed a message $S = (s_1, \dots, s_H)$ consisting of $H$ symbols, where each symbol $s_h \in \{0, \dots, M-1\}$ and $M=2^m$.

At each step $t$, we derive pseudo-random parameters from a secret key $K$ and a local context window $g_t$ (the previous $w$ tokens). 
Specifically, we generate a position allocation index $i_t \in \{1, \dots, H\}$ that determines which message symbol $s_{i_t}$ to embed at the current step, alongside a keyed permutation $\phi_t$ that maps this symbol to a specific target bin index on the CDF.

The permutation $\phi_t$ ensures that a fixed symbol value is not systematically bound to the same quantile bin across steps, which is essential for the unbiasedness guarantees in Section~\ref{sec:unbiasedness}.






\subsection{Embedding via Quantile Partitioning}
\label{sec:embed}

The core of our embedding strategy is to enforce a uniform probability budget for every symbol by operating on the continuous domain of the quantile function.
This requires partitioning the unit interval and calculating the mass overlap to bridge the continuous design with the discrete vocabulary.

\textbf{Quantile Geometry and Overlap Mass.}
At step $t$, we sort the vocabulary by probability to map each token $v$ to a sub-interval $I_t(v) \subset [0,1)$ of length $p_0(v|h_t)$.
We simultaneously partition $[0,1)$ into $M$ fixed bins $B_r = [r/M, (r+1)/M)$.
The alignment between token intervals and bins is captured by the \textbf{overlap mass}:
\begin{equation}
    \mu_t(v, r) = \left| I_t(v) \cap B_r \right|.
\end{equation}

This explicitly quantifies the mass of $v$ falling within $B_r$, bridging the continuous design with the discrete vocabulary and enabling the use of partial mass from tokens that straddle bin boundaries.

\textbf{Bin-Restricted Sampling.}
To embed the target symbol $s_{i_t}$, we identify the target bin $r^\star_t = \phi_t(s_{i_t})$.
Instead of standard sampling (drawing $u \in [0,1)$ globally), QuantileMark samples a random value $u_t$ strictly within the target bin $B_{r^\star_t}$.
The output token $x_t$ is determined by finding the token whose interval $I_t(x_t)$ contains $u_t$.
The resulting watermarked distribution for a chosen bin $r$ is:
\begin{equation}
p_{\mathrm{wm}}(v\mid h_t,r) = M \cdot \mu_t(v,r).
\label{eq:wm-dist}
\end{equation}

This sampling mechanism guarantees that every target symbol receives an identical probability budget $1/M$ regardless of context entropy.
Consequently, this structurally enforces message symmetry while effectively exploiting it: in contrast to minimal-distortion baselines like StealthInk, we concentrate the entire probability mass into the target bin, substantially boosting the stepwise evidence available for detection.
\subsection{Decoding and Detection}
\label{sec:detect}

The detection process follows a \emph{decode-then-test} paradigm: we first reconstruct the quantile geometry to recover the most likely message $\hat{S}$, and then compute a detection score based on the confidence of this recovery.

\textbf{Stepwise Evidence Extraction.}
The core of our detection strategy is to recover soft evidence about the embedding channel using teacher forcing.
At each step $t$, the detector reconstructs the base distribution $p_0(\cdot \mid h_t)$ and the quantile geometry.
The observed token $x_t$ occupies the interval $I_t(x_t)$.
Given that the sampling value must lie within this interval, the posterior probability that it originated from bin $r$ is the proportion of the token's mass overlapping that bin:
\begin{equation}
    p_t(r) = P(R_t = r \mid x_t, h_t) = \frac{\mu_t(x_t, r)}{p_0(x_t \mid h_t)}.
    \label{eq:bin-posterior}
\end{equation}

Crucially, this posterior provides soft evidence.
Unlike vocabulary partitioning, which forces a hard assignment (casting a vote for a single symbol and rejecting the other $M-1$), $p_t(r)$ distributes support proportional to the overlap mass.
For example, a token straddling bins $r_1$ and $r_2$ yields positive support for both while decisively ruling out non-overlapping bins.

To aggregate this evidence numerically, we compute the Log Posterior Odds (LPO):
\begin{equation}
    \mathrm{LPO}_t(r) = \log \left( \frac{p_t(r)}{1 - p_t(r)} \right).
\end{equation}
For numerical stability, we clip $p_t(r)$ to the range $[\epsilon, 1-\epsilon]$.

\textbf{Message Decoding.}
Since the position allocation $i_t$ is deterministic given the key, we aggregate evidence for each message symbol $s_h$ independently.
The decoder identifies the symbol $\hat{s}_h$ that maximizes the cumulative LPO across all steps assigned to index $h$:
\begin{equation}
\hat{s}_h = \operatorname*{arg\,max}_{s \in \{0, \dots, M-1\}} \sum_{t : i_t = h} \mathrm{LPO}_t\big(\phi_t(s)\big).
    \label{eq:decoding}
\end{equation}
Here, $\phi_t(s)$ maps the candidate symbol $s$ to its corresponding bin index at step $t$.

\textbf{Sequence-Level Scoring.}
Once the message $\hat{S} = (\hat{s}_1, \dots, \hat{s}_H)$ is decoded, we evaluate the overall presence of the watermark.
We define the implied target bin path as $\hat{r}_t = \phi_t(\hat{s}_{i_t})$ and compute the average evidence along this path: $T(x_{1:T}) = \frac{1}{T} \sum_{t=1}^{T} \mathrm{LPO}_t(\hat{r}_t)$

The watermark is detected if $T(x_{1:T}) > \theta$, where $\theta$ is a threshold calibrated to control the false positive rate on non-watermarked text.
\subsection{Unbiasedness Properties of Equal-Mass Channelization}
\label{sec:unbiasedness}
We assume a uniform prior over the message space. Under this assumption, the target symbol $s$ at any given step is uniformly distributed over $[M]$.
\begin{definition}[Message-unbiasedness at a step]
\label{def:msg-unbiased-step}
Fix a context $h_t$ and key $K$.
Let $s$ be a target symbol drawn uniformly from $[M]$, and let $R=\phi_t(s)$ be the relabeled bin index.
The scheme is message-unbiased if, for all tokens $v$,
\[
\mathbb{E}_{s}\big[p_{\mathrm{wm}}(v\mid h_t,R)\big]=p_0(v\mid h_t).
\]
\end{definition}

\begin{lemma}[Equal-mass bins imply message-unbiasedness]
\label{lem:msg-unbiased}
Assume $\phi_t$ maps a uniform symbol to a uniform bin in $[M]$,
then QuantileMark is message-unbiased at step $t$.
\end{lemma}

The proof follows from the linearity of expectation and the partition of unity property of the quantile bins; see Appendix~\ref{sec:app_unbiased} for details.

Message-unbiasedness ensures that, on average, the watermark introduces no distortion if the message is unknown.
Crucially, this property serves as a structural pathway to satisfying the dual notion of \textbf{cipher-unbiasedness} ~\cite{jiang2025stealthink}, which requires the distribution to be preserved when the message is fixed but the key is randomized.

\begin{definition}[Cipher-unbiasedness at a step]
\label{def:cipher-unbiased-step}
Fix a context $h_t$, a key $K$, and a symbol $s\in[M]$.
Let $Z$ be a seed random variable, and let $\Phi_{t,Z}$ denote the resulting key-conditioned permutation on $[M]$ produced from $Z$.
The scheme is cipher-unbiased if, for all tokens $v$,
\[
\mathbb{E}_{Z}\!\left[p_{\mathrm{wm}}\!\left(v\mid h_t,\Phi_{t,Z}(s)\right)\right]
=
p_0(v\mid h_t).
\]
\end{definition}

QuantileMark satisfies Definition~\ref{def:cipher-unbiased-step} whenever, for any fixed $s$, the induced bin index $\Phi_{t,Z}(s)$ is uniform on $[M]$ under the seed distribution.
The proof is given in Appendix~\ref{sec:app_geometry} (Lemma~\ref{lem:cipher-unbiased}).

\begin{table*}[t]
\small
\centering
\setlength{\tabcolsep}{2.8pt}
\renewcommand{\arraystretch}{1.1}
\begin{tabular}{lcccccccccc}
\toprule
& \multicolumn{5}{c}{C4} &
  \multicolumn{5}{c}{LFQA} \\
\cmidrule(lr){2-6} \cmidrule(lr){7-11}
Method &
\makecell{Bit Acc $\uparrow$} &
\makecell{AUC  $\uparrow$} &
\makecell{TPR@ \\ 1\%FPR $\uparrow$} &
\makecell{PPL $\downarrow$} &
\makecell{Time (s) $\downarrow$} &
\makecell{Bit Acc $\uparrow$} &
\makecell{AUC  $\uparrow$} &
\makecell{TPR@ \\ 1\%FPR $\uparrow$} &
\makecell{PPL  $\downarrow$} &
\makecell{Time (s) $\downarrow$} \\
\midrule
No watermark        & --    & -- & -- & 7.684  & -- & --    & -- & -- & 2.647 & -- \\
MPAC                & 0.9702 & 0.9887 & 0.9800 & 10.351 & \textbf{0.103} & 0.8770 & 0.9756 & 0.8056 & 3.793 & \textbf{0.111} \\
StealthInk          & 0.9003 & 0.9869 & 0.7920  & 8.235  & 0.645 & 0.7447 & 0.8301 & 0.2425 & \textbf{2.636} & 1.013 \\
\textbf{QuantileMark} (ours) & \textbf{0.9893} & \textbf{0.9995} & \textbf{0.9840} & \textbf{7.404} & 0.343 & \textbf{0.9500} & \textbf{0.9997} & \textbf{1.0} & 2.759 & 0.348 \\
\bottomrule
\end{tabular}
\caption{Generation and detection performance on C4 and LFQA with 24 bits embedded in 300 tokens. Time (s) denotes average detection time per sample.}
\label{tab:detection-main}
\end{table*}
\section{Experiments}
\label{sec:experiments}

\begin{figure*}[t]
  \centering
  \definecolor{myblue}{RGB}{31, 119, 180}
  \definecolor{myorange}{RGB}{255, 127, 14}
  \definecolor{mygreen}{RGB}{44, 160, 44}
  
  \pgfplotsset{
    myplotstyle/.style={
      width=\linewidth,
      height=4.25cm,
      grid=both,
      grid style={line width=.1pt, draw=gray!20},
      major grid style={line width=.2pt,draw=gray!40},
      xlabel={Tokens $T$},
      xmin=50, xmax=450,
      xtick={50,100,150,200,250,300,350,400,450},
      legend style={
        font=\scriptsize,
        at={(0.98,0.02)},      
        anchor=south east,     
        cells={anchor=west},
        draw=gray!50,
        rounded corners=1pt,
        fill=white, fill opacity=0.65,
        inner sep=0.25pt,
        row sep=-1pt,
      },
      tick label style={font=\footnotesize},
      label style={font=\small},
      title style={font=\small, yshift=-1ex, align=center},
      cycle list name=color list,
      every axis plot/.append style={very thick, mark size=2.5pt}
    }
  }

  \begin{minipage}{0.48\textwidth}
    \centering
    \begin{tikzpicture}
      \begin{axis}[
        myplotstyle,
        ylabel={Bit accuracy},
        ymin=0.55, ymax=1.0,
        title={(a) Message Recovery},
      ]

      \addplot[color=myblue, mark=*, solid] coordinates {
        (50,0.5953798199287889)
        (100,0.7295918373429046)
        (150,0.8299319738028)
        (200,0.8929988667875731)
        (250,0.9281854365059133)
        (300,0.9506541514966268)
        (350,0.9678612046274309)
        (400,0.9758248005710771)
        (450,0.98421501645456)
      };
      \addlegendentry{QuantileMark}

      \addplot[color=myorange, mark=triangle*, densely dashdotted] coordinates {
        (50,0.6309931523179355)
        (100,0.7051940635867315)
        (150,0.758847031691303)
        (200,0.8016895755254936)
        (250,0.843642608509031)
        (300,0.873997708161672)
        (350,0.8966954013396954)
        (400,0.9160919530638333)
        (450,0.9295976996421814)
      };
      \addlegendentry{MPAC}

      \addplot[color=mygreen, mark=square*, dashed] coordinates {
        (50,0.5851934545096897)
        (100,0.6269900517677194)
        (150,0.668662674412756)
        (200,0.7014720563938518)
        (250,0.7209331315613079)
        (300,0.7427644709983986)
        (350,0.767215568773047)
        (400,0.7809381240499234)
        (450,0.7904116466461893)
      };
      \addlegendentry{StealthInk}

      \end{axis}
    \end{tikzpicture}
  \end{minipage}
  \hfill
  \begin{minipage}{0.48\textwidth}
    \centering
    \begin{tikzpicture}
      \begin{axis}[
        myplotstyle,
        ylabel={Detection AUC},
        ymin=0.5, ymax=1.02,
        title={(b) Presence Detection },
        legend style={at={(0.98,0.02)}, anchor=south east}
      ]

      \addplot[color=myblue, mark=*, solid] coordinates {
        (50,0.7893354620759868)
        (100,0.962561895506502)
        (150,0.9953028830579851)
        (200,0.9992132907584803)
        (250,0.9997437360947711)
        (300,0.9998718680473855)
        (350,0.9998951647660427)
        (400,0.9999184614846999)
        (450,0.999941758203357)
      };
      \addlegendentry{QuantileMark}

      \addplot[color=myorange, mark=triangle*, densely dashdotted] coordinates {
        (50,0.7500996903734285)
        (100,0.8470749671608182)
        (150,0.9094518202289359)
        (200,0.9396086489295119)
        (250,0.9630613715001004)
        (300,0.9795172470802186)
        (350,0.9835255648038048)
        (400,0.9856420927467301)
        (450,0.9903686087990486)
      };
      \addlegendentry{MPAC}

      \addplot[color=mygreen, mark=square*, dashed] coordinates {
        (50,0.5583368764172336)
        (100,0.6120338605480062)
        (150,0.6862203736240094)
        (200,0.7353840223744127)
        (250,0.7979176377783356)
        (300,0.8313268672236365)
        (350,0.8524328588332318)
        (400,0.8697066943956399)
        (450,0.879767564232835)
      };
      \addlegendentry{StealthInk}

      \end{axis}
    \end{tikzpicture}
  \end{minipage}

  \caption{Varying detection length $T$ on Llama-3.1-8B-Instruct for LFQA.}
  \label{fig:lfqa_token_sweep}
  \vspace{-2mm}
\end{figure*}
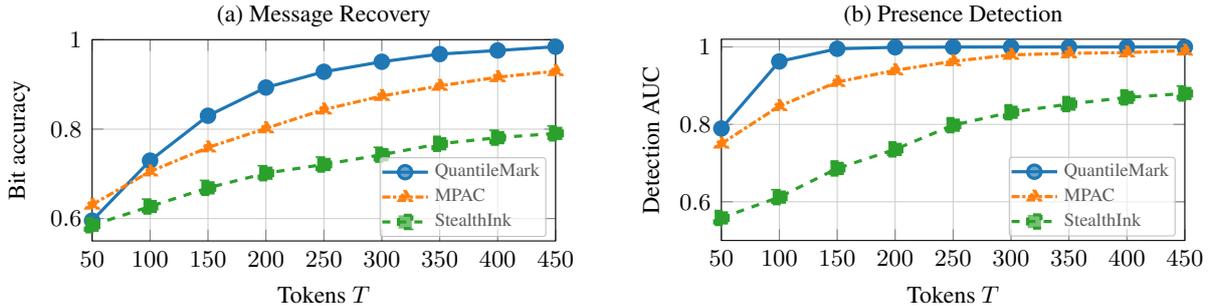
We evaluate QuantileMark in a provider-internal setting, focusing on accurate \emph{multi-bit} recovery from clean generations and robustness against realistic deployment mismatches.
Our results confirm that QuantileMark significantly improves detection stability over vocabulary-based baselines while maintaining generation quality.

\subsection{Experimental Setup}
\label{sec:exp-setup}

We conduct experiments on two tasks: open-ended continuation on C4~\cite{raffel2020exploring} using Llama-2-7B~\cite{touvron2023llama}, and long-form QA on ELI5/LFQA~\cite{eli5_lfqa} using Llama-3.1-8B-Instruct~\cite{grattafiori2024llama3herdmodels}.
We pair each task with a suitable model, using a base model for continuation and an instruction-tuned model for QA, to ensure the evaluation reflects realistic deployment behavior.
Unless otherwise stated, results are averaged over 500 watermarked and 500 non-watermarked samples per task, all fixed to a length of $T=300$ tokens.
We use human references as non-watermarked text to simulate realistic detection against organic text.
Appendix~\ref{sec:app_model_negatives} additionally evaluates against non-watermarked model generations to isolate the watermark signal from distributional mismatch.

We embed a 24-bit message using top-$k{=}128$ sampling and temperature $\tau{=}1.0$.
We compare QuantileMark ($m{=}2$) against MPAC~\cite{yoo2023advancing} ($m{=}2, \gamma{=}0.25, \delta{=}2.0$) and StealthInk~\cite{jiang2025stealthink} ($m{=}1$), following the recommended settings for best trade-offs.
Detection of QuantileMark reconstructs geometry via teacher forcing on output tokens (adding chat template headers for Instruct models) with matching top-$k/\tau$; detector mismatch is analyzed separately in Section~\ref{subsec:lfqa-mismatch}.
We report Bit Acc (bit-wise accuracy, averaged over 500 watermarked samples), detection AUC, TPR@1\%FPR, and perplexity (PPL).

\begin{table*}[t]
\centering
\small
\setlength{\tabcolsep}{4pt}
\renewcommand{\arraystretch}{1.08}
\begin{tabular}{l cccccccccc}
\toprule
& \multicolumn{2}{c}{No attack} &
  \multicolumn{2}{c}{Copy-paste ($\epsilon=0.2$)} &
  \multicolumn{2}{c}{Synonym ($\epsilon=0.2$)} &
  \multicolumn{2}{c}{Deletion ($\epsilon=0.1$)} &
  \multicolumn{2}{c}{Paraphrase (Dipper)} \\
\cmidrule(lr){2-3} \cmidrule(lr){4-5} \cmidrule(lr){6-7} \cmidrule(lr){8-9} \cmidrule(lr){10-11}
Method &
Bit Acc $\uparrow$ & AUC $\uparrow$ &
Bit Acc $\uparrow$ & AUC $\uparrow$ &
Bit Acc $\uparrow$ & AUC $\uparrow$ &
Bit Acc $\uparrow$ & AUC $\uparrow$ &
Bit Acc $\uparrow$ & AUC $\uparrow$ \\
\midrule
MPAC          & 0.9702 & 0.9887 & 0.9438 & 0.9849 & 0.9476 & 0.9876 & \textbf{0.8811} & \textbf{0.9750} & 0.7201 & \textbf{0.7936} \\
StealthInk    & 0.9003 & 0.9869 & 0.8497 & 0.9550 & 0.8608 & 0.9701 & 0.7839 & 0.8854 & 0.6829 & 0.6536 \\
\textbf{QuantileMark} & \textbf{0.9893} & \textbf{0.9995} & \textbf{0.9730} & \textbf{0.9972} & \textbf{0.9712} & \textbf{0.9963} & 0.8712 & 0.9426 & \textbf{0.7640} & 0.7609 \\
\bottomrule
\end{tabular}
\vspace{-1mm}
\caption{
Robustness on C4 with 24 bits embedded in 300 tokens.
The attack intensities $\epsilon$ denote the fraction of tokens modified.
}
\label{tab:robustness}
\vspace{-2mm}
\end{table*}
\subsection{Generation and Detection Results}
\label{subsec:detection}

Table~\ref{tab:detection-main} summarizes the generation and detection performance of different methods on two tasks.
QuantileMark achieves consistently strong multi-bit recovery and near-saturated detection, with high bit accuracy and AUC and TPR@1\%FPR close to 1.0 on both C4 and LFQA.
Notably, these gains persist on LFQA, where the next-token distribution is often more peaked and vocabulary-partition methods can allocate highly uneven probability mass across symbol values, destabilizing detection.

In terms of generation quality, QuantileMark maintains PPL close to the unwatermarked baseline; minor fluctuations (e.g., 7.404 vs 7.684 on C4) fall within expected finite-sample randomness, consistent with our theoretical message-unbiasedness.
In contrast, MPAC incurs a noticeable PPL increase due to distortionary green-list bias, while StealthInk shows weaker recovery at the same message length.
This performance comes with negligible operational overhead: although detection requires a model forward pass, the computational cost is strictly bounded by the generation latency itself, making verification inherently affordable for any provider capable of hosting the service.

Varying the generated length $T$ (Figure~\ref{fig:lfqa_token_sweep}) confirms that QuantileMark accumulates evidence efficiently, reaching high recovery and detection rates with fewer tokens than baselines.

We next vary the symbol size $m$ while keeping the message length fixed at 24 bits.
Figure~\ref{fig:ablation_m_final} shows a trade-off between symbol resolution and reliable evidence.
Moderate choices ($m=2,3$) perform best, while $m=4$ sharply reduces recovery even though detection remains strong.
Intuitively, since a single observed token reveals limited information about the latent bin (Appendix~\ref{sec:app_dilution}), increasing $m$ under a fixed-length budget fails to provide sufficient supporting tokens for $M$-ary decisions. 
This effectively dilutes the information per symbol, explaining the drop in bit accuracy observed at $m=4$.

The lower AUC at $m=1$ stems from the sharp CDF geometry of human text. Human tokens often fall in the tail, yielding nearly deterministic bin posteriors. With binary partitions ($M=2$), the detector's maximization step can easily find a spurious message that fortuitously aligns with these strong signals. Increasing to $M=4$ makes such chance agreement statistically much harder—as each token supports a specific bin while rejecting three others—thereby improving separation.

\begin{figure}[t]
\centering
\begin{tikzpicture}

    \definecolor{c1}{RGB}{222, 235, 247} 
    \definecolor{c2}{RGB}{158, 202, 225}
    \definecolor{c3}{RGB}{66, 146, 198}
    \definecolor{c4}{RGB}{8, 69, 148}

    \begin{axis}[
        ybar,
        width=\linewidth,
        height=4.6cm,
        symbolic x coords={Bit Accuracy, Detection AUC},
        xtick=data,
        bar width=14pt, 
        enlarge x limits=0.45, 
        ymin=0, 
        ymax=1.03, 
        ylabel={Score},
        ytick={0,  0.25, 0.5, 0.75, 1.0},
        axis x line*=bottom,
        axis y line*=left,
        grid=major,
        grid style={dotted, gray!30},
        legend style={
            at={(0.5, 1.1)}, 
            anchor=south,
            legend columns=-1,
            draw=none, 
            fill=none,
            font=\small
        },
        tick label style={font=\small},
        label style={font=\small},
    ]
        \addplot[style={fill=c1, draw=c4!40, thin}] coordinates {
            (Bit Accuracy, 0.90915) (Detection AUC, 0.71464)
        };
        \addlegendentry{$m=1$}
        \addplot[style={fill=c2, draw=c4!40, thin}] coordinates {
            (Bit Accuracy, 0.94998) (Detection AUC, 0.99974)
        };
        \addlegendentry{$m=2$}
        \addplot[style={fill=c3, draw=c4!40, thin}] coordinates {
            (Bit Accuracy, 0.957) (Detection AUC, 0.99885)
        };
        \addlegendentry{$m=3$}
        \addplot[style={fill=c4, draw=none}] coordinates {
            (Bit Accuracy, 0.53941) (Detection AUC, 0.99799)
        };
        \addlegendentry{$m=4$}

    \end{axis}
\end{tikzpicture}
\caption{LFQA ablation of $m$ for QuantileMark with 24 bits embedded. We vary the bits per symbol $m$ ( $M=2^m$ quantile bins). }
\label{fig:ablation_m_final}
\end{figure}
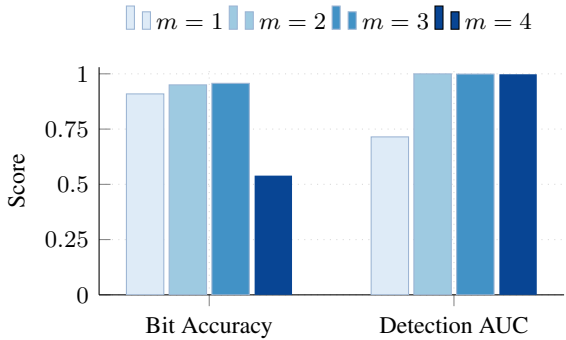
\subsection{Robustness to Scrubbing Attacks on C4}
\label{subsec:robustness}

We evaluate robustness on C4 under four scrubbing attacks commonly used in prior multi-bit watermark evaluations~\cite{jiang2025stealthink}:
copy-paste mixing with non-watermarked text, synonym substitution, random deletion, and paraphrasing.
Table~\ref{tab:robustness} reports bit accuracy and detection AUC; full attack specifications are described in Appendix~\ref{sec:app_attacks}.

While the watermark resists local lexical edits (copy-paste and substitution), deletions and paraphrasing prove significantly more damaging. These attacks induce synchronization drift and degrade the context quality required to reconstruct the quantile geometry, thereby hindering the recovery of long multi-bit messages.
Robust detection under paraphrasing remains an open challenge across all existing multi-bit watermarking schemes.

\subsection{Detector-Generator Mismatch on LFQA}
\label{subsec:lfqa-mismatch}

We evaluate detector-generator mismatch when text is generated with $(k{=}128,\tau{=}1.0)$ but detector varies $(\hat{k},\hat{\tau})$.
Figure~\ref{fig:lfqa_mismatch_combined} shows that performance degrades smoothly under moderate mismatch:
bit accuracy remains high across $\hat{\tau}\in[0.8,1.2]$ and remains stable even for small detector top-$\hat{k}$ (e.g., $\hat{k}=8$).
This robustness is expected in our white-box setting because the detector reconstructs quantile geometry from logits:
moderate changes in truncation or temperature often preserve the probability ordering on the head of the distribution, so overlap masses and channel posteriors shift gradually rather than catastrophically.

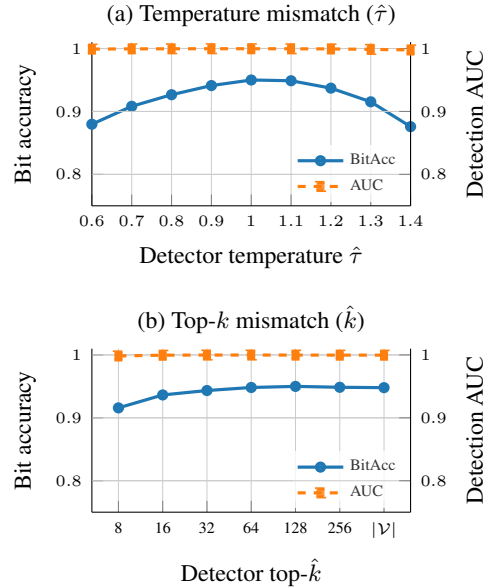
\begin{figure}[t]
  \centering
  \definecolor{tempblue}{RGB}{31,119,180}
  \definecolor{temporange}{RGB}{255,127,14}

  \pgfplotsset{
    commonstyle/.style={
      width=0.75\linewidth,
      height=3.75cm,
      grid=both,
      grid style={line width=.1pt, draw=gray!20},
      major grid style={line width=.2pt, draw=gray!40},
      tick label style={font=\tiny},
      label style={font=\footnotesize},
      legend style={
        font=\tiny,
        at={(0.98,0.02)}, 
        anchor=south east,
        cells={anchor=west},
        draw=none,        
        fill=white, fill opacity=0.8, 
        inner sep=1pt,
      },
      every axis plot/.append style={very thick, mark size=1.5pt},
    }
  }

  \begin{subfigure}[b]{\linewidth}
    \centering
    \begin{tikzpicture}
      \begin{axis}[
        commonstyle,
        xlabel={Detector temperature $\hat{\tau}$},
        xmin=0.6, xmax=1.4,
        xtick={0.6,0.7,0.8,0.9,1.0,1.1,1.2,1.3,1.4},
        ylabel={Bit accuracy},
        ymin=0.75, ymax=1.01,
        axis y line*=left,
        axis x line*=bottom,
        title={(a) Temperature mismatch ($\hat{\tau}$)},
        title style={font=\footnotesize, yshift=-0.8ex},
      ]
        \addplot[color=tempblue, mark=*] coordinates {
          (0.6,0.87968) (0.7,0.90815) (0.8,0.9266) (0.9,0.94113) (1.0,0.94998)
          (1.1,0.94881) (1.2,0.93704) (1.3,0.91533) (1.4,0.87584)
        };
        \addlegendentry{BitAcc}

        \addlegendimage{color=temporange, mark=square*, dashed, very thick, mark size=1.5pt}
        \addlegendentry{AUC}
      \end{axis}

      \begin{axis}[
        commonstyle,
        ylabel={Detection AUC},
        ymin=0.75, ymax=1.01,
        axis y line*=right,
        axis x line=none,
        yticklabel pos=right,
        xmin=0.6, xmax=1.4,
      ]

        \addplot[color=temporange, mark=square*, dashed] coordinates {
          (0.6,0.99912) (0.7,0.99951) (0.8,0.99961) (0.9,0.99971) (1.0,0.99973)
          (1.1,0.9996) (1.2,0.99942) (1.3,0.99869) (1.4,0.99782)
        };
      \end{axis}
    \end{tikzpicture}
    \label{fig:sub_temp}
  \end{subfigure}
  
  \vspace{0.2cm}

  \begin{subfigure}[b]{\linewidth}
    \centering
    \begin{tikzpicture}
      \begin{axis}[
        commonstyle,
        symbolic x coords={8,16,32,64,128,256,Full},
        xtick=data,
        xticklabels={8,16,32,64,128,256,$|\mathcal{V}|$}, 
        xlabel={Detector top-$\hat{k}$},
        ylabel={Bit accuracy},
        ymin=0.75, ymax=1.01,
        axis y line*=left,
        axis x line*=bottom,
        title={(b) Top-$k$ mismatch ($\hat{k}$)},
        title style={font=\footnotesize, yshift=-0.8ex},
      ]

        \addplot[color=tempblue, mark=*] coordinates {
          (8,0.91583) (16,0.93637) (32,0.94339) (64,0.94823) (128,0.94998) (256,0.94856) (Full,0.94798)
        };
        \addlegendentry{BitAcc}

        \addlegendimage{color=temporange, mark=square*, dashed, very thick, mark size=1.5pt}
        \addlegendentry{AUC}
      \end{axis}

      \begin{axis}[
        commonstyle,
        symbolic x coords={8,16,32,64,128,256,Full},
        ylabel={Detection AUC},
        ymin=0.75, ymax=1.01,
        axis y line*=right,
        axis x line=none,
        yticklabel pos=right,
      ]

        \addplot[color=temporange, mark=square*, dashed] coordinates {
          (8,0.99838) (16,0.99955) (32,0.99981) (64,0.99983) (128,0.99973) (256,0.99967) (Full,0.99965)
        };
      \end{axis}
    \end{tikzpicture}
    \label{fig:sub_topk}
  \end{subfigure}

  \caption{Detector-generator mismatch analysis on Llama-3.1-8B-Instruct.}
  \label{fig:lfqa_mismatch_combined}
  \vspace{-3mm}
\end{figure}
\begin{table}[t]
\centering
\small
\setlength{\tabcolsep}{1pt} 
\renewcommand{\arraystretch}{1.05}
\begin{tabular}{@{}lccc|ccc@{}} 
\toprule
& \multicolumn{3}{c}{Machine translation} &
\multicolumn{3}{c}{Summarization} \\
\cmidrule(lr){2-4}\cmidrule(lr){5-7}
Method
& BLEU $\uparrow$ & R-1 $\uparrow$ & BERT $\uparrow$
& R-1 $\uparrow$ & BERT $\uparrow$ & PPL $\downarrow$ \\
\midrule
No watermark
& 17.55 & 49.13 & 35.82
& 37.42 & 21.40 & 5.18 \\
\midrule
MPAC
& 16.89 & 47.98 & 34.89
& 36.81 & 20.57 & 5.81 \\
StealthInk
& \textbf{17.42} & 48.97 & 35.75
& 37.08 & 20.83 & 5.72 \\
QuantileMark
& 17.41 & \textbf{48.99} & \textbf{35.99}
& \textbf{37.44} & \textbf{21.37} & \textbf{5.21} \\
\bottomrule
\end{tabular}
\vspace{-1mm}
\caption{Downstream generation quality under top-$k{=}128$ sampling. For machine translation, we report BLEU, ROUGE-1 (R-1), and BERTScore (BERT). For summarization, we report ROUGE-1, BERTScore, and perplexity (PPL).}
\label{tab:quality}
\vspace{-2mm}
\end{table}
\subsection{Utility on Downstream Tasks}
\label{subsec:quality}

Finally, we assess whether watermarking degrades downstream utility under top-$k{=}128$ sampling.
For text summarization, we use BART-large~\cite{lewis2020bart} evaluated on CNN/DailyMail~\cite{hermann2015teaching}.
For machine translation, we run Multilingual BART (mBART)~\cite{liu2020multilingual} on the WMT'14 En-Ro corpus~\cite{bojar-EtAl:2014:W14-33}.
We also include a reference-free LFQA evaluation using GPT-4o, with the full experimental setup and results detailed in Appendix~\ref{sec:app_gpt4o_judge}.

Table~\ref{tab:quality} shows that QuantileMark consistently outperforms other watermarking baselines, achieving utility scores that are nearly identical to the no-watermark upper bound across both tasks.

\section{Related Work}
\label{sec:related}

We focus on generative watermarking schemes that modify the sampling process, distinguishing between distortionary partitioning and distribution-preserving methods.

\subsection{Distortionary Vocabulary Partitioning Multi-bit Watermarks}

A dominant line of work constructs a channel by partitioning the vocabulary and biasing the next token distribution toward a key dependent subset.
The watermark of \citet{kirchenbauer2023watermark} introduces green list promotion with statistical tests on set hits, followed by analyses of calibration and robustness under edits~\cite{piet2023mark,kirchenbauer2023reliability}.
Multi bit extensions retain the same mechanism while adding structure across positions, for example via position allocation and multiple colorings in MPAC~\cite{yoo2023advancing}, or via coding designs over partition assignments~\cite{wang2023towards,fernandez2023three}.
Several works improve robustness by adding error correcting codes from the view of message encoding~\cite{chao2025rbc,qu2024provably}.

However, standard partitioning schemes inherently suffer from probability mass imbalance, particularly in peaked distributions.
While recent heuristics like rank-based partitioning~\cite{park2025watermod} or entropy-based gating~\cite{gu2025invisibleentropy} mitigate this issue, they remain adaptive approximations.
In contrast, QuantileMark addresses this structurally by redefining channels in CDF interval, guaranteeing an equal probability budget by construction rather than adaptive rejection.

\subsection{Distribution-preserving and Unbiased Watermarks}

A parallel line of work aims to preserve the base model distribution in expectation over the watermark randomness.
Two complementary perspectives predominate in the literature: \emph{RNG-space} constructions, which define a measure-preserving transformation on the sampling randomness~\cite{kuditipudi2023robust}, and \emph{reweighting} constructions, which formulate watermarking as a randomized modification of the stepwise token distribution~\cite{hu2023unbiased}.

\textbf{RNG-space distortion-free watermarks.}
\citet{kuditipudi2023robust} propose distortion-free watermarking by mapping a key-derived random number sequence to samples from the language model, and instantiate the framework with inverse transform sampling and exponential minimum sampling. SynthID-Text introduces Tournament sampling~\citep{deepmind2024synthid}.
Multi-bit extensions encode information by transforming the sampling randomness based on the message: DISC employs cyclic shifts~\cite{boroujeny2024multi}, while MirrorMark utilizes measure-preserving mirroring~\citep{mirrormark2025}.

\textbf{Expectation-unbiased reweighting.}
\citet{hu2023unbiased} formalize unbiased watermarking as randomized reweighting of the stepwise distribution.
For autoregressive generation, preserving the sequence distribution requires appropriate independence of the watermark codes across steps~\cite{hu2023unbiased}.
DiPmark~\citep{wu2023dipmark} and StealthInk~\citep{jiang2025stealthink} further explore distribution-preserving designs, with StealthInk emphasizing text-only detection and multi-bit provenance.

QuantileMark shares the design of defining watermarks through structured sampling randomness but leverages white-box access to optimize for detection reliability rather than black-box stealth.
Unlike prior schemes that rely on randomized vocabulary partitions, QuantileMark enforces equal-mass channelization on the cumulative probability interval.
This structural design substantially boosts the stepwise statistical evidence available for decoding, all while maintaining the theoretical unbiasedness of the generation process.

\section{Conclusion}
\label{sec:conclusion}

We presented \textbf{QuantileMark}, a white-box multi-bit watermark designed to mitigate the challenges of provenance embedding in low-entropy regimes.
Instead of relying on vocabulary partitioning, QuantileMark defines the channel in cumulative probability interval by dividing it into $M$ equal-mass bins.
This structural approach ensures that every symbol receives a fixed probability budget, providing a basis for \emph{message symmetry} where messages are treated with consistent statistical weight.
Complementing this embedding, we derived a teacher-forced detector that exploits the reconstructed quantile geometry to compute posteriors over latent bins, enabling a coherent decode-then-test process.
Empirical results across base and instruction-tuned models demonstrate the feasibility and effectiveness of this design, showing improved recovery rates and detection performance compared to vocabulary-based baselines, while preserving generation quality.
QuantileMark thus offers a practically viable solution for provider-internal attribution that balances message symmetry with operational utility.

\section*{Limitations}

Our work has several limitations that suggest directions for future research.
First, robustness remains limited under heavy paraphrasing and other edits that substantially rewrite local contexts, since such changes can disrupt position allocation and the seed stream used for per-step relabeling and evidence aggregation. Improving paraphrase resilience may require stronger position allocation strategies that tolerate sequence misalignment and a PRF or seeding logic that depends less on fragile local token neighborhoods while remaining reproducible for the key holder.
Second, our study focused on a provider-internal setting and assumed white-box access at detection time, which rules out public verification from text alone in the current form. Although detection may transfer to settings with a proxy model, a distilled model, or a closely related model family sharing the same tokenizer, we did not systematically study how model mismatch and distribution shift affect the reconstructed geometry and downstream recovery.
Finally, our experimental evaluation was limited in scope, utilizing a small set of tasks and model families. Extending evaluation across broader decoding policies settings would better characterize the generality of quantile-based channelization for multi-bit provenance.

\section*{Ethical Considerations}
\label{sec:ethic}
The main ethical goal of multi-bit watermarking is accountability. It allows providers to trace malicious outputs, such as disinformation, back to the responsible user.
While this naturally raises privacy concerns about user surveillance, the risk is limited because the watermark is passive. It only matters if a user shares the generated text publicly; if the text remains private, the watermark is harmless.
Furthermore, our provider-internal design protects user privacy better than black-box schemes. The embedded ID is meaningless to third parties. Only the model provider can extract the signal and link it to a real user. Ultimately, while providers have this tracking ability, they must use it strictly for safety and compliance, guided by strong privacy policies, rather than for invasive profiling.

\section*{Acknowledgements}
This work was supported by Beijing Natural Science Foundation (L253001), Key Laboratory of Science, Technology and Standard in Press Industry (Key Laboratory of Intelligent Press Media Technology) and National Engineering Research Center of New Electronic Publishing Technologies. We appreciate the anonymous reviewers for their helpful comments. Xiaojun Wan is the contact author.

\bibliography{custom}

\appendix

\section{Channel Geometry and Message-Unbiasedness}
\label{sec:app_geometry}

This appendix collects auxiliary results that complement
Section~\ref{sec:embed}--\ref{sec:unbiasedness}.
We reuse the stepwise channel geometry from the main text but switch to a context-level notation by dropping the step index.
Fix a context $h$ and its decoding distribution $p_0(\cdot\mid h)$.
Define the discrete CDF $F_h$, token CDF intervals $I_h(y)$, equal-width quantile bins $\{B_r\}_{r\in[M]}$, and overlap mass
\[
\mu_h(y,r)\coloneqq |I_h(y)\cap B_r|
\]
exactly as in Section~\ref{sec:embed}.

\subsection{Basic Overlap Identities}
\label{sec:app_overlap}

\begin{lemma}[Basic overlap identities]
\label{lem:overlap-identities}
For every context $h$:
\begin{align}
\sum_{r=0}^{M-1} \mu_h(y,r)
&= |I_h(y)| = p_0(y\mid h),
\qquad \forall y\in\mathcal{V}, \label{eq:mu-sum-r}\\
\sum_{y\in\mathcal{V}} \mu_h(y,r)
&= |B_r| = \frac{1}{M},
\qquad \forall r\in[M]. \label{eq:mu-sum-y}
\end{align}
\end{lemma}

\begin{proof}
Equation~\eqref{eq:mu-sum-r} holds because $\{B_r\}_r$ partitions $[0,1)$.
Equation~\eqref{eq:mu-sum-y} holds because $\{I_h(y)\}_y$ partitions $[0,1)$.
\renewcommand{\qedsymbol}{}
\end{proof}

\subsection{Unbiasedness Notions: Message and Cipher}
\label{sec:app_unbiased}

We recall two notions introduced in the main text.
Message-unbiasedness is defined in Definition~\ref{def:msg-unbiased-step},
and cipher-unbiasedness (a StealthInk-style viewpoint) is defined in
Definition~\ref{def:cipher-unbiased-step}.
In this appendix we use a context-level notation by dropping the step index:
fix a context $h$ and the corresponding base distribution $p_0(\cdot\mid h)$.
The overlap mass $\mu_h(y,r)$ is defined exactly as in
Section~\ref{sec:embed}, and the per-channel distribution satisfies
\begin{equation}
p_{\mathrm{wm}}(y\mid h,r)=M\,\mu_h(y,r).
\label{eq:app_wm_dist}
\end{equation}

We now give the proof of Lemma~\ref{lem:msg-unbiased}, which is stated in the main text.

\begin{proof}[Proof of Lemma~\ref{lem:msg-unbiased}]
Fix $h$.
By assumption, $\phi$ maps a uniform symbol to a uniform bin, so
$R=\phi(s)$ is uniform on $[M]$ when $s$ is uniform on $[M]$.
Therefore, for any token $y\in\mathcal{V}$,
\begin{align}
\mathbb{E}_{s}\!\left[p_{\mathrm{wm}}(y\mid h,R)\right]
&=\mathbb{E}_{R}\!\left[p_{\mathrm{wm}}(y\mid h,R)\right] \notag\\
&=\frac{1}{M}\sum_{r=0}^{M-1} p_{\mathrm{wm}}(y\mid h,r) \notag\\
&=\frac{1}{M}\sum_{r=0}^{M-1} M\,\mu_h(y,r) \notag\\
&=\sum_{r=0}^{M-1}\mu_h(y,r) \notag\\
&=p_0(y\mid h),
\end{align}
where the last equality uses~\eqref{eq:mu-sum-r}.
\renewcommand{\qedsymbol}{}
\end{proof}

\begin{lemma}[Uniform cipher-induced relabeling implies cipher-unbiasedness]
\label{lem:cipher-unbiased}
Fix $(h,s)$.
Assume that under the seed prior in Definition~\ref{def:cipher-unbiased-step},
the key-derived permutation $\Phi$ induces
$R=\Phi(s)$ that is uniform on $[M]$.
Then for every token $y\in\mathcal{V}$,
\begin{equation}
\mathbb{E}\!\left[
p_{\mathrm{wm}}\!\left(y\mid h,\Phi(s)\right)
\right]
=
p_0(y\mid h).
\end{equation}
\end{lemma}

\begin{proof}
Under the assumption, $R$ is uniform on $[M]$ even with $s$ fixed. Hence
\begin{align}
\mathbb{E}\!\left[p_{\mathrm{wm}}\!\left(y\mid h,\Phi(s)\right)\right]
&= \mathbb{E}_{R}\!\left[p_{\mathrm{wm}}(y\mid h,R)\right] \notag\\
&= \frac{1}{M}\sum_{r=0}^{M-1} p_{\mathrm{wm}}(y\mid h,r) \notag\\
&= \frac{1}{M}\sum_{r=0}^{M-1} M\,\mu_h(y,r) \notag\\
&= \sum_{r=0}^{M-1}\mu_h(y,r)
= p_0(y\mid h),
\end{align}
where the last step uses~\eqref{eq:mu-sum-r}.
\renewcommand{\qedsymbol}{}
\end{proof}
\subsection{Signal Dilution and Information Capacity}
\label{sec:app_dilution}

This subsection analyzes a fundamental decoding limitation: the information recoverable about the bin index is capped by the token surprisal $-\log_2 p_0(y\mid h)$, regardless of watermark capacity $m$.

\paragraph{Posterior Bound.}
Fix a context $h$ and observed token $y$.
Assuming a uniform prior $\mathbb{P}(R=r)=1/M$, the posterior for bin $r$ is given by Eq.~\eqref{eq:channel-posterior}: $\mathbb{P}(R=r \mid y,h) = \mu_h(y,r)/p_0(y\mid h)$.
Since the overlap mass satisfies $\mu_h(y,r) \le |B_r| = 1/M$, the posterior is uniformly bounded:
\begin{equation}
\label{eq:posterior-ub}
\max_{r\in[M]} \mathbb{P}(R=r\mid y,h) \;\le\; \min\!\left\{1, \frac{1}{M p_0(y\mid h)}\right\}.
\end{equation}
Intuitively, if a token's probability mass $p_0$ spans many bins ($M p_0 \gg 1$), the probability is diluted across all contained bins, preventing confident identification of any single bin.

\paragraph{Information Cap.}
The information gain $\Delta(y) \coloneqq H(R) - H(R\mid y,h)$ is bounded by the min-entropy. Using \eqref{eq:posterior-ub}:
\begin{equation}
\label{eq:info-cap}
\Delta(y) \le \min\!\left\{\log_2 M,\ -\log_2 p_0(y\mid h)\right\}.
\end{equation}
This bound implies that high-probability tokens cannot carry more bits of watermark information than their own self-information.
For instance, if $p_0(y\mid h) \ge 0.5$, observing $y$ yields at most 1 bit of information about the bin location, regardless of how large $M$ is.
Reliable multi-bit recovery therefore relies on aggregating partial evidence across the sequence rather than identifying the bin from a single step.
\subsection{Time-Averaged Bin Occupancy under a Fixed Message}
\label{sec:app_bin_occupancy}

This subsection models the empirical occupancy of quantile bins along a long continuation.
Fix a message $s_{1:H}\in[M]^H$ and consider a length-$T$ generated continuation $x_{1:T}$.
At step $t$, the embedding rule selects a target bin index
\[
r_t^\star \coloneqq \phi_t(s_{i_t}) \in [M],
\]
where $(i_t,\phi_t)$ are derived from the key $K$ and the step seed (e.g., $z_t=\mathrm{Hash}(g_t)$ as in the main text).
Define the empirical bin occupancy
\begin{equation}
\hat{\pi}_T(r)\coloneqq \frac{1}{T}\sum_{t=1}^T \mathbf{1}\{r_t^\star=r\},
\qquad r\in[M].
\label{eq:pi-hat}
\end{equation}
When $\hat{\pi}_T$ is close to uniform, the scheme exhibits a channel-hopping effect over time:
although each step samples within a single bin, the visited bins cover $[M]$ nearly evenly at the sequence level.

We formalize this effect under an idealized PRF model, conditioning on the realized sequence of step seeds.
Let $\mathcal{Z}_T$ be the set of distinct seeds appearing among $\{z_t\}_{t=1}^T$.
For each $z\in\mathcal{Z}_T$, define its multiplicity
\begin{align}
n_z &\coloneqq \big|\{t\in[T] : z_t=z\}\big|,
\qquad
\sum_{z\in\mathcal{Z}_T} n_z = T.
\label{eq:nz-def}
\end{align}
Because PRF outputs are functions of $z$, repeated seeds reuse the same $(i_t,\phi_t)$.
Thus, for each distinct seed $z$, define the induced bin label
\[
R_z \coloneqq \phi_{z}(s_{i_z}) \in [M],
\]
where $(i_z,\phi_z)$ denotes the PRF-derived pair associated with seed $z$.
Abbreviate
\begin{equation}
S_T \coloneqq \sum_{z\in\mathcal{Z}_T} n_z^2,
\qquad
T_{\mathrm{eff}} \coloneqq \frac{T^2}{S_T}.
\label{eq:teff}
\end{equation}
Here $T_{\mathrm{eff}}$ acts as an effective sample size: repeated seeds inflate $S_T$ and reduce concentration.

\begin{proposition}[Key-random occupancy is near-uniform, with an effective sample size]
\label{prop:bin-occupancy}
Assume an idealized PRF model in which, conditional on the realized seed sequence $(z_1,\dots,z_T)$,
the pairs $\{(i_z,\phi_z)\}_{z\in\mathcal{Z}_T}$ are independent across distinct seeds,
and each $\phi_z$ is a uniform random permutation on $[M]$ independent of $i_z$.
Then for any fixed message $s_{1:H}$ and any $r\in[M]$,
\begin{equation}
\mathbb{E}\!\left[\hat{\pi}_T(r)\mid (z_1,\dots,z_T)\right] = \frac{1}{M}.
\label{eq:pi-unbiased-key}
\end{equation}
Moreover, for any $\epsilon>0$, let $\mathbf z\coloneqq (z_1,\dots,z_T)$.

\begin{equation}
\begin{aligned}
\mathbb{P}\!\left(
\left|\hat{\pi}_T(r)-\frac{1}{M}\right|\ge \epsilon
\,\middle|\, \mathbf z
\right)
&\le 2\exp\!\left(-\frac{2\epsilon^2 T^2}{S_T}\right)\\
&=2\exp\!\left(-2\epsilon^2 T_{\mathrm{eff}}\right).
\end{aligned}
\label{eq:pi-hoeffding}
\end{equation}

\end{proposition}

\begin{proof}
Condition on $(z_1,\dots,z_T)$.
Grouping repeated seeds rewrites~\eqref{eq:pi-hat} as a weighted sum over distinct seeds:
\begin{equation}
\hat{\pi}_T(r)
=
\sum_{z\in\mathcal{Z}_T} \frac{n_z}{T}\,\mathbf{1}\{R_z=r\}.
\label{eq:pi-weighted}
\end{equation}
Under the permutation assumption, for any fixed symbol $d\in[M]$ the image $\phi_z(d)$ is uniform on $[M]$,
so $R_z=\phi_z(s_{i_z})$ is uniform on $[M]$ and
$\mathbb{E}[\mathbf{1}\{R_z=r\}\mid (z_1,\dots,z_T)]=1/M$, proving~\eqref{eq:pi-unbiased-key}.

For concentration, let $X_z\coloneqq \mathbf{1}\{R_z=r\}\in\{0,1\}$ and weights $w_z\coloneqq n_z/T$.
The variables $\{X_z\}_{z\in\mathcal{Z}_T}$ are independent and $w_z(X_z-\mathbb{E}X_z)\in[-w_z,w_z]$.
Applying the standard weighted Hoeffding inequality yields
\begin{equation}
\begin{alignedat}{1}
&\mathbb{P}\!\left(
\left|\sum_{z\in\mathcal{Z}_T} w_z\!\left(X_z-\frac{1}{M}\right)\right|
\ge \epsilon \,\middle|\, \mathbf z
\right) \\
&\le 2\exp\!\left(
-\frac{2\epsilon^2}{\sum_{z\in\mathcal{Z}_T} w_z^2}
\right).
\end{alignedat}
\end{equation}

Substituting $\sum_z w_z^2=\sum_z (n_z/T)^2=S_T/T^2$ gives~\eqref{eq:pi-hoeffding}.
\renewcommand{\qedsymbol}{}
\end{proof}
This proposition highlights a purely \emph{key-driven} channel-hopping effect. Even with a fixed message $s_{1:H}$, the target mapping $r_t^\star$ visits bins nearly uniformly over time. This uniform coverage is critical for preventing \emph{structural bias}, ensuring that a deterministic message does not continuously target a small subset of bins, which would otherwise induce a perceptible distributional shift.

Crucially, this uniformity ($\hat{\pi}_T(r)\approx 1/M$) describes the distribution of the \emph{target labels}, not necessarily the token statistics of unwatermarked text. The PRF construction achieves a dual goal: it generates a target distribution that is statistically near-uniform (preventing the systematic overuse of any single bin) while remaining fully deterministic and reproducible for the key holder (enabling decoding). Finally, the effective sample size $T_{\mathrm{eff}}$ accounts for the redundancy caused by seed collisions: repeated seeds force the reuse of PRF outputs, thereby reducing the number of independent samples and slowing the convergence to uniformity.

\section{Detection as a Composite GLRT with LPO Evidence}
\label{sec:app_glrt}

This appendix formalizes the detection process as a Generalized Likelihood Ratio Test (GLRT) and justifies the use of Log Posterior Odds (LPO) as the stepwise evidence metric.

\subsection{Hypotheses and the Composite Alternative}
Let $x_{1:T}$ be an observed continuation and $h_t=x_{<t}$.
We test
\begin{align}
H_0 &: x_{1:T}\sim p_0, \\
H_1 &: x_{1:T}\sim p_{\mathrm{wm}}(\cdot\mid s)
\text{ for some message } s\in\mathcal{S}.
\end{align}
We assume the detector has white-box access to the model and can reconstruct $p_0(\cdot\mid h_t)$ via teacher forcing.
All PRF-derived assignments are reproducible because they are derived per step from the secret key and a hash of a local
context $g_t=x_{t-w:t-1}$ in the generated stream.

\subsection{Sequence Likelihood under a Fixed Message}
At each step, a message symbol determines a target channel index $r_t^\star\in[M]$ (after relabeling),
and the per-step watermarked distribution is $p_{\mathrm{wm}}(x_t\mid h_t,s)=M\,\mu_{h_t}(x_t,r_t^\star)$.
Using the standard autoregressive factorization, we write
\begin{equation}
\label{eq:seq-like-factor}
\begin{aligned}
p_{\mathrm{wm}}(x_{1:T}\mid s)
&= \prod_{t=1}^T p_{\mathrm{wm}}(x_t\mid h_t,s) \\
&=\; \prod_{t=1}^T M\,\mu_{h_t}(x_t,r_t^\star),\\
p_0(x_{1:T})
&= \prod_{t=1}^T p_0(x_t\mid h_t).
\end{aligned}
\end{equation}

\subsection{Decode-then-Test Equals a GLRT}
The generalized likelihood ratio statistic for $H_0$ versus the composite alternative $H_1$ is
\begin{equation}
\Lambda_{\mathrm{GLRT}}(x_{1:T})
\coloneqq
\frac{\max_{s\in\mathcal{S}} p_{\mathrm{wm}}(x_{1:T}\mid s)}{p_0(x_{1:T})}.
\end{equation}
Let $\hat s\in\arg\max_{s\in\mathcal{S}}\log p_{\mathrm{wm}}(x_{1:T}\mid s)$ be the decoded message.
Then the decode--then--test score
\begin{equation}
T_{\mathrm{GLRT}}(x_{1:T})
\coloneqq
\log p_{\mathrm{wm}}(x_{1:T}\mid \hat s) - \log p_0(x_{1:T})
\end{equation}
equals $\log \Lambda_{\mathrm{GLRT}}(x_{1:T})$ up to tie-breaking.

\subsection{Channel Posterior and Token Evidence}
Introduce a latent channel variable $R_t\in[M]$ indicating which quantile bin generated token $x_t$.
Given teacher-forced reconstruction of $p_0(\cdot\mid h_t)$ and the overlap geometry,
\begin{equation}
\label{eq:channel-posterior}
\mathbb{P}(R_t=r \mid x_t,h_t)
=
\frac{\mu_{h_t}(x_t,r)}{p_0(x_t\mid h_t)}.
\end{equation}
We define per-token log-posterior odds (LPO) evidence for the channel event $R_t=r$:
\begin{equation}
\mathrm{LPO}_t(r)
\coloneqq
\log\frac{\mathbb{P}(R_t=r\mid x_t,h_t)}{1-\mathbb{P}(R_t=r\mid x_t,h_t)}.
\end{equation}

To avoid overloading watermark-level hypotheses, define the event-level hypotheses
$G_1(r): R_t=r$ and $G_0(r): R_t\neq r$.
Assume a uniform prior $\mathbb{P}(R_t=r)=1/M$.

\begin{lemma}[Token LPO as an event-level likelihood ratio up to a constant]
\label{lem:lpo-glrt}
Let the mixture under $G_0(r)$ be
\[
p(x_t\mid G_0(r))=\frac{1}{M-1}\sum_{j\neq r} p(x_t\mid R_t=j).
\]
Then
\begin{equation}
\mathrm{LPO}_t(r)
=
\log\frac{p(x_t\mid G_1(r))}{p(x_t\mid G_0(r))}
-\log(M-1).
\end{equation}
\end{lemma}

\begin{proof}
By Bayes' rule,
\[
\begin{aligned}
&\frac{\mathbb{P}(R_t=r\mid x_t,h_t)}{\mathbb{P}(R_t\neq r\mid x_t,h_t)} \\
&\quad=
\frac{p(x_t\mid R_t=r)\,\mathbb{P}(R_t=r)}
{\sum_{j\neq r} p(x_t\mid R_t=j)\,\mathbb{P}(R_t=j)} \\
&\quad=
\frac{p(x_t\mid R_t=r)}
{\sum_{j\neq r} p(x_t\mid R_t=j)} \\
&\quad=
\frac{1}{M-1}\cdot
\frac{p(x_t\mid R_t=r)}{p(x_t\mid G_0(r))}.
\end{aligned}
\]
Taking logs yields the claim.
\renewcommand{\qedsymbol}{}
\end{proof}

\subsection{Comparison: LPO vs. LLR}
\label{sec:llr-lpo-comparison}

For a fixed decoded message, the standard Log-Likelihood Ratio (LLR) evidence is:
\begin{equation}
\mathrm{LLR}_t(r) \coloneqq \log \frac{p_{\mathrm{wm}}(x_t\mid h_t, r)}{p_0(x_t\mid h_t)} = \log M + \log p_t(r).
\label{eq:llr-def}
\end{equation}
While LLR and LPO are functionally related via $\mathrm{LPO}_t(r) = \mathrm{LLR}_t(r) - \log M - \log(1-p_t(r))$, LPO offers distinct numerical advantages.

\paragraph{Symmetry and Dynamic Range.}
We implement a clipped posterior $p_t(r) \in [\epsilon, 1-\epsilon]$ (with $\epsilon=10^{-6}$) for stability.
\begin{itemize}
    \item \textbf{LLR} is asymmetric and capped from above: $\mathrm{LLR}_t(r) \le \log M$. Even if the token is perfectly aligned ($p_t(r) \approx 1$), the positive evidence is limited by the bit-width $m$.
    \item \textbf{LPO} is symmetric and unbounded (before clipping): $\mathrm{LPO}_t \in [-C, C]$ where $C \approx 13.8$. This allows near-certain alignments to contribute significantly more evidence than $\log M$, improving separation when $p_t(r) \to 1$.
\end{itemize}

\paragraph{Empirical Validation.}
Table~\ref{tab:lpo-llr-compare} and Figure~\ref{fig:ablation_m_llr} compare decoding using LPO versus LLR.
While bit accuracy is similar (indicating decoding is driven by the rank order of posteriors), LPO yields higher detection AUC and better separation on C4.

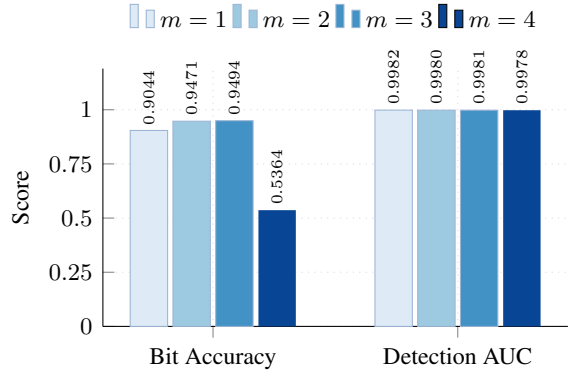
\begin{figure}[t]
\centering
\begin{tikzpicture}
    \definecolor{c1}{RGB}{222, 235, 247} 
    \definecolor{c2}{RGB}{158, 202, 225}
    \definecolor{c3}{RGB}{66, 146, 198}
    \definecolor{c4}{RGB}{8, 69, 148}
    \begin{axis}[
        ybar,
        width=\linewidth,
        height=5cm,
        symbolic x coords={Bit Accuracy, Detection AUC},
        xtick=data,
        bar width=14pt, 
        enlarge x limits=0.45, 
        ymin=0, 
        ymax=1.19, 
        ylabel={Score},
        ytick={0,  0.25, 0.5, 0.75, 1.0},
        axis x line*=bottom,
        axis y line*=left,
        grid=major,
        grid style={dotted, gray!30},
        legend style={
            at={(0.5, 1.1)}, 
            anchor=south,
            legend columns=-1,
            draw=none, 
            fill=none,
            font=\small
        },
        nodes near coords,
        every node near coord/.append style={
            font=\tiny, 
            rotate=90, 
            anchor=west,
            inner sep=2pt,
            /pgf/number format/fixed,
            /pgf/number format/fixed zerofill,
            /pgf/number format/precision=4
        },
        tick label style={font=\small},
        label style={font=\small},
    ]
        \addplot[style={fill=c1, draw=c4!40, thin}] coordinates {
            (Bit Accuracy, 0.90439) (Detection AUC, 0.99818)
        };
        \addlegendentry{$m=1$}

        \addplot[style={fill=c2, draw=c4!40, thin}] coordinates {
            (Bit Accuracy, 0.94706) (Detection AUC, 0.99804)
        };
        \addlegendentry{$m=2$}
        \addplot[style={fill=c3, draw=c4!40, thin}] coordinates {
            (Bit Accuracy, 0.9494) (Detection AUC, 0.99806)
        };
        \addlegendentry{$m=3$}
        \addplot[style={fill=c4, draw=none}] coordinates {
            (Bit Accuracy, 0.53641) (Detection AUC, 0.99777)
        };
        \addlegendentry{$m=4$}
    \end{axis}
\end{tikzpicture}
\caption{Performance comparison on LFQA with varying bits per symbol ($m$) using LLR as the token-level evidence. This serves as a complement to Figure~\ref{fig:ablation_m_final}, which uses LPO.}
\label{fig:ablation_m_llr}
\end{figure}
\begin{table}[t]
\centering
\scriptsize
\setlength{\tabcolsep}{3pt}
\renewcommand{\arraystretch}{1.12}
\begin{tabular}{lcccc}
\toprule
Evidence & \makecell{Bit Acc $\uparrow$} & \makecell{Detection AUC $\uparrow$} & \makecell{Score on wm\\($\mu \pm \sigma$)} & \makecell{Score on no wm\\($\mu \pm \sigma$)} \\
\midrule
\multicolumn{5}{l}{\textbf{C4}} \\
\cmidrule(lr){1-5}
LPO & 0.9868 & 0.9989 & 2.215 $\pm$ 1.134 & -2.791 $\pm$ 0.531 \\
LLR & 0.9879 & 0.9751 & -0.9069 $\pm$ 0.414 & -3.336 $\pm$ 0.659\\
\midrule
\multicolumn{5}{l}{\textbf{LFQA}} \\
\cmidrule(lr){1-5}
LPO & 0.9500 & 0.9997 & -0.046 $\pm$ 0.949 & -3.152 $\pm$ 0.410 \\
LLR & 0.9470 & 0.9980 & -0.899 $\pm$ 0.359 & -4.291 $\pm$ 0.554 \\
\bottomrule
\end{tabular}
\vspace{-1mm}
\caption{
Comparison of LPO and LLR evidence within QuantileMark ($m=2, H=12$) on C4 and LFQA.
Score entries report mean $\pm$ standard deviation.
}
\label{tab:lpo-llr-compare}
\vspace{-2mm}
\end{table}

\section{Relation to RNG Space and Reweighting Viewpoints}
\label{sec:app_rng}

Section~\ref{sec:related} surveys distribution preserving watermarking from the RNG space and reweighting viewpoints.
This appendix clarifies how QuantileMark fits into the broader landscape of distribution-preserving watermarks, specifically comparing the handling of randomness and reweighting mechanisms.

\subsection{RNG-Space Discretization and Unbiasedness}
\label{sec:app_rng_structure}

Standard distortion-free schemes typically define a measure-preserving map on the sampling randomness $u_t \in [0,1)$, ensuring $x_t \sim p_0$ for \emph{every} fixed message.
In contrast, QuantileMark discretizes the RNG space $[0,1)$ into $M$ disjoint bins and restricts sampling to a specific bin $B_{r^\star}$.
Conceptually, this partitions the entropy source into a watermark-controlled variable (the bin index) and residual stochasticity (uniform sampling within the bin).

While this restriction implies the distribution is distorted for a \emph{fixed} message, Lemma~\ref{lem:msg-unbiased} proves that the scheme recovers $p_0$ when marginalized over the message or key.
This design prioritizes \textbf{message symmetry} over strict per-message stealth. 
In provider-internal settings, this trade-off is advantageous: it maximizes the statistical evidence for detection while maintaining fairness and unbiasedness on average across the user base.
\subsection{Vocabulary Shuffling versus CDF Message Relabeling}
\label{sec:app_rng_shuffle_vs_relabel}

Some unbiased reweighting rules rely on shuffling a vocabulary order before applying accept amplify style operations ~\cite{hu2023unbiased,jiang2025stealthink}.
QuantileMark does not require vocabulary shuffling.
It keeps the CDF geometry fixed by the model induced probability ordering and uses only a low dimensional relabeling on $[M]$.

At each step, the PRF produces a message index assignment $i_t$ and a permutation $\phi_t$ on $[M]$.
The permutation relabels message symbols into bin indices, and the bins themselves are defined by equal mass partitioning of the stepwise CDF.
The randomized object is thus an element of $[M]$ or a permutation on $[M]$, rather than a randomized mask over $\mathcal{V}$.
This keeps the embedding rule simple and makes reconstruction straightforward in a white box setting, since the detector can rebuild the same quantile geometry from logits and evaluate overlap masses directly.

\section{Detailed Experiment Setup}
\label{sec:app_detailed_exp}

This appendix provides additional details on data preprocessing, filtering protocols, and the configuration of robustness attacks used in Section~\ref{sec:experiments}.

\subsection{Dataset Details}
\label{sec:app_dataset_details}

\paragraph{C4 (Open-ended continuation).}
We utilize the \texttt{realnewslike} subset of the C4 dataset~\cite{raffel2020exploring}.
For each example, we truncate the first 50 tokens to serve as the prompt and treat the subsequent text as the human reference.
To ensure valid comparisons, we first filter the dataset to retain only examples where the human reference exceeds the target evaluation length $T$.

\paragraph{LFQA (Long-form Question Answering).}
We use the ELI5/LFQA dataset~\cite{eli5_lfqa}, where the provided questions serve directly as prompts for the model.
Similar to C4, we filter for human answers that exceed length $T$.

\subsection{Attack Specifications}
\label{sec:app_attacks}

We evaluate robustness against four types of post-hoc editing attacks.
For the first three attacks, the parameter $\epsilon$ controls the intensity of the distortion.

\begin{itemize}
    \item \textbf{Copy-Paste Mixing:}
    This attack simulates a scenario where watermarked content is embedded within a larger non-watermarked context.
    We construct the attacked text by mixing a fraction $1-\epsilon$ of the watermarked generation with a fraction $\epsilon$ of non-watermarked text.
    The watermarked segment is inserted contiguously to simulate a copy-paste operation.

    \item \textbf{Synonym Substitution:}
    We replace a fraction $\epsilon$ of the tokens in the watermarked text with their synonyms using a predefined synonym dictionary, while preserving the original sentence structure.

    \item \textbf{Random Deletion:}
    We randomly select and remove a fraction $\epsilon$ of tokens from the watermarked sequence.

    \item \textbf{Paraphrasing:}
    We employ DIPPER~\cite{krishna2023paraphrasing}, a paraphrase generation model designed for controlling lexical and syntactic diversity.
    Following standard evaluation protocols, we configure DIPPER with lexical diversity $L=20$ and order diversity $O=0$.
\end{itemize}

\begin{table}[t]
\centering
\small
\setlength{\tabcolsep}{2.5pt}
\renewcommand{\arraystretch}{1.1}
\begin{tabular}{lcccc}
\toprule
& \multicolumn{2}{c}{C4} & \multicolumn{2}{c}{LFQA} \\
\cmidrule(lr){2-3} \cmidrule(lr){4-5}
Method &
AUC $\uparrow$ &
\makecell{TPR@\\1\%FPR $\uparrow$} &
AUC $\uparrow$ &
\makecell{TPR@\\1\%FPR $\uparrow$} \\
\midrule
MPAC       & 0.9993  & 0.986 & 0.9744 & 0.8337 \\
StealthInk & 0.9851 & 0.6960 & 0.8081 & 0.1824 \\
\makecell{\textbf{QuantileMark}}
           & \textbf{0.9998} & \textbf{0.9900} & \textbf{0.9985} & \textbf{0.9499} \\
\bottomrule
\end{tabular}
\caption{Detection performance on C4 and LFQA with 24 bits embedded in 300 tokens. Negatives are non-watermarked text generated by model.}
\label{tab:detection-app}
\end{table}
\section{Additional Results}
\label{sec:app_additional}
\subsection{Model Generation as Non-Watermarked Text}
\label{sec:app_model_negatives}

In the main text, we use human references as detection negatives (non-watermarked text) to reflect a deployment setting where provenance is verified among organic text.
Since our detector is model-assisted and reconstructs token-level geometry via teacher forcing, the choice of negatives can interact with distribution mismatch between human text and the model.
To isolate the watermark signal from such mismatch, we additionally evaluate detection where negatives are unwatermarked generations from the same base model, produced with the same prompts and decoding parameters as the watermarked texts.
We report AUC and TPR at 1\% FPR under this matched-model control, as Table~\ref{tab:detection-app} shows.

\subsection{GPT-4o Judging for LFQA}
\label{sec:app_gpt4o_judge}

We employ GPT-4o~\cite{openai2024gpt4ocard} as a reference-free judge for LFQA. The evaluator assesses the generated text directly and assigns a single integer score on a scale of 1 to 5. Following the protocol established by MirrorMark \citep{mirrormark2025}, we design the prompt to focus exclusively on linguistic quality while explicitly ignoring artifacts resulting from truncation. Table~\ref{tab:lfqa_gpt4o} shows the result.

\begin{table}[t]
\small
\centering
\setlength{\tabcolsep}{6pt}
\renewcommand{\arraystretch}{1.08}
\begin{tabular}{lc}
\toprule
Method & GPT-4o $\uparrow$ \\
\midrule
No watermark & 4.115 \\
\midrule
MPAC &  3.921\\
StealthInk & 4.047 \\
QuantileMark & \textbf{4.109} \\
\bottomrule
\end{tabular}
\vspace{-1mm}
\caption{
LFQA answer quality judged by GPT-4o.
Scores are on a 1--5 scale; higher is better.
We randomly sample 500 prompts from LFQA and generate answers with a maximum length of $T{=}300$ new tokens.
During evaluation, we exclude outputs shorter than 50 tokens to mitigate potential length-related bias.
}
\label{tab:lfqa_gpt4o}
\vspace{-2mm}
\end{table}

\begin{judgebox}
\small
You are a strict and consistent text-quality evaluator.
Use ONLY the given text; do NOT assume the author's intent.
The text may start or end abruptly because the generation length is fixed. Do NOT penalize truncation or incompleteness.
Do NOT judge factual correctness.

Rate each field as an integer from 1 to 5.
\textbf{overall} is an independent judgment; do NOT compute overall from the other fields (no arithmetic).

Rate the following text.
Return only a JSON object in exactly the following structure:

\begin{verbatim}
{
  "coherence": int,
  "clarity": int,
  "naturalness": int,
  "overall": int
}
\end{verbatim}

Text:
[TEXT]
\end{judgebox}

\subsection{Scaling to Longer Messages}
\label{sec:app_longer_msg}

We evaluate QuantileMark with longer messages embedded in a fixed sequence of $T{=}300$ tokens on LFQA, using $m{=}2$ bits per symbol.

\begin{table}[h]
\centering
\small
\setlength{\tabcolsep}{6pt}
\renewcommand{\arraystretch}{1.08}
\begin{tabular}{ccc}
\toprule
Message length (bits) & Bit Acc $\uparrow$ & AUC $\uparrow$ \\
\midrule
24 & 0.9500 & 0.9997 \\
32 & 0.9196 & 0.9997 \\
40 & 0.9025 & 0.9997 \\
48 & 0.8801 & 0.9994 \\
56 & 0.8690 & 0.9995 \\
\bottomrule
\end{tabular}
\caption{Effect of increasing message length on LFQA ($T{=}300$, $m{=}2$). Detection AUC remains near-perfect, while bit accuracy degrades gracefully as each symbol receives fewer supporting tokens.}
\label{tab:longer-msg}
\end{table}

These results demonstrate that QuantileMark scales well to longer messages, maintaining robust detection performance while degrading gracefully in bit recovery accuracy. As the message grows, each symbol is allocated fewer supporting tokens, reducing the per-symbol evidence budget.

\subsection{Mitigating Repeated Seeds}
\label{sec:app_repeat_seed}

As noted in Section~\ref{sec:setup-pa}, the per-step parameters $(i_t, \phi_t)$ are derived from a hash of the local context window $g_t$.
When consecutive steps produce identical context hashes (seed collisions), the same bin assignment is reused, which can introduce sentence-level distributional bias even though per-step message-unbiasedness holds.

A simple mitigation is to \textbf{skip embedding} at any step whose seed has already appeared in the current sequence, treating the skipped step as an erasure.
Table~\ref{tab:repeat-seed} evaluates this strategy on LFQA with 24 bits embedded in 300 tokens.

\begin{table}[h]
\centering
\small
\setlength{\tabcolsep}{5pt}
\renewcommand{\arraystretch}{1.08}
\begin{tabular}{lccc}
\toprule
Method & Bit Acc $\uparrow$ & AUC $\uparrow$ & PPL $\downarrow$ \\
\midrule
QuantileMark             & 0.9500 & 0.9997 & 2.759 \\
\quad + skip repeated seed & 0.9438 & 0.9996 & 2.644 \\
\bottomrule
\end{tabular}
\caption{Effect of skipping repeated seeds on LFQA. Skipping slightly reduces PPL and mitigates sentence-level bias, while maintaining detection performance.}
\label{tab:repeat-seed}
\end{table}

As expected, skipping reduces perplexity (confirming that sentence-level bias is mitigated) while detection performance remains essentially unchanged, thanks to the robust evidence aggregation of the detector.

\end{document}